\documentclass[aos, preprint]{imsart}
\RequirePackage[OT1]{fontenc}
\RequirePackage{amsthm,amsmath,amsfonts}
\RequirePackage[numbers]{natbib}
\RequirePackage[colorlinks,citecolor=blue,urlcolor=blue]{hyperref}
\usepackage{algorithm}
\usepackage{graphicx}
\usepackage{multirow}
\usepackage{pdflscape}
\usepackage{enumitem}
\usepackage{tabularx}
\usepackage{rotating}
\usepackage{bbm}
\usepackage{xr}

\newcommand{\bR}{\mathbb R}
\newcommand{\bP}{\mathbb {P}}
\newcommand{\bE}{\mathbb {E}}
\newcommand{\bV}{\mathbb {V}}

\newcommand{\mM}{\mathcal M}

\newcommand{\fa}{{\mathbf{ a}}}
\newcommand{\fb}{{\boldsymbol{ b}}}

\newcommand{\fd}{{\boldsymbol{ d}}}
\newcommand{\fe}{{\mathbf{ e}}}

\newcommand{\fr}{{\boldsymbol{ r}}}

\newcommand{\fv}{{ \boldsymbol{\nu}}}

\newcommand{\fx}{{\mathbf{ x}}}
\newcommand{\fy}{{\mathbf{ y}}}

\newcommand{\fD}{{       { D}}}

\newcommand{\fI}{{       { I}}}

\newcommand{\fL}{{       { L}}}

\newcommand{\fR}{{       { R}}}
\newcommand{\fS}{{       { S}}}

\newcommand{\fU}{{       { U}}}

\newcommand{\fW}{{       { W}}}
\newcommand{\fX}{{       { X}}}

\newcommand{\fSigma}{{ \Sigma}}
\newcommand{\fzero}{\mathbf{0}}
\newcommand{\mA}{\mathcal A}
\newcommand{\mJ}{\mathcal J}

\newcommand{\mhA}{\tilde {\mathcal A}}
\newcommand{\mAo}{\mathcal A}

\newcommand{\mG}{\mathcal G}

\newcommand{\mB}{\mathcal B}

\newcommand{\Psic}{\Psi^{\bot}}

\newcommand{\hbeta}{{\hat  \fbeta}}
\newcommand{\hbetaat}[1]{{\hat  \fbeta_{\tilde \mA{#1}}}}
\newcommand{\hbetact}[1]{{\hat  \fbeta_{\tilde \mA^c{#1}}}}

\newcommand{\hbetag}{{\hat  \fbeta_{\ftau\mG}}}
\newcommand{\hbetagc}{{\hat  \fbeta_{\ftau\mG^c}}}

\newcommand{\fbetao}{{\fbeta^*}}
\newcommand{\fbetaat}{{\fbeta^*_{\tilde \mA}}}
\newcommand{\fbetact}{{\fbeta^*_{\tilde \mA^c}}}

\newcommand{\fbetag}{{\fbeta_{0,\mG}}}

\newcommand{\fbetaobb}{{\fbeta_{0,\mB_2}}}

\newcommand{\fSaa}{{       {\fS}_{\mhA\mhA}}}

\newcommand{\fSca}{{       {\fS}_{\mhA^c \mhA}}}
\newcommand{\fScc}{{       {\fS}_{\mhA^c\mhA^c}}}

\newcommand{\fSpsiaa}{{       {\Psi \fS \Psi}}}
\newcommand{\fSpsiac}{{       {\Psi \fS \Psic}}}
\newcommand{\fSpsica}{{       {\Psic \fS \Psi}}}
\newcommand{\fSpsicc}{{       {\Psic \fS \Psic}}}

\newcommand{\fapa}{{       {\Psi \fa}}}
\newcommand{\fapc}{{       {\Psic \fa}}}

\newcommand{\fSigmaaa}{{  { \hat \Sigma}_{\mhA\mhA}}}
\newcommand{\ifSigmaaa}{{ { \hat \Sigma}^{-1}_{\mhA\mhA}}}

\newcommand{\fSigmacc}{{  { \hat \Sigma}_{\mhA^c\mhA^c}}}
\newcommand{\ifSigmacc}{{ { \hat \Sigma}^{-1}_{\mhA^c\mhA^c}}}

\newcommand{\fSigmagg}{  { \hat \Sigma_{\mG\mG}}}
\newcommand{\fSigmalsl}{ { \hat \Sigma_{LSL}}}

\newcommand{\ftaua}{{             \ftau_{a}}}
\newcommand{\ftauc}{{             \ftau_{c}}}

\newcommand{\fkappa}{{\boldsymbol \kappa}}
\newcommand{\fbeta}{{ \boldsymbol   \beta}}
\newcommand{\ftau}{{ \boldsymbol \tau}}

\newcommand{\fepsilon}{{\boldsymbol  \epsilon}}

\newcommand{\supnorm}[1]{{ \| {#1} \|_{\infty}}}
\newcommand{\onenorm}[2]{{ \| {#1} \|_{1}^{#2}}}
\newcommand{\twonorm}[2]{{ \| {#1} \|_{2}^{#2}}}

\newcommand{\maxnorm}[1]{{  | {#1}  |_{\infty}}}

\newcommand{\tr}[1]{ \text{tr} ( {#1} )}
\newcommand{\fnormt}[2]{{ \| {#1} \|_{F}^{#2}}}
\newcommand{\fSigmatau}[1]{  { \hat \Sigma}_{#1}}
\newcommand{\ifSigmatau}[1]{ { \hat \Sigma}^{-1}_{#1}}

\newcommand{\hbetatau}[1]{{ \hat \fbeta_{#1} }}

\newcommand{\rhomin}[1]{{ \rho^{+}_{\text{min}}{(#1)}}}
\newcommand{\rhomax}[1]{{ \rho^{+}_{\text{max}}{(#1)}}}

\newcommand{\hL}{{\fL}}
\newcommand{\hLpaa}[1]{ \fL_{\Psi\Psi}^{#1}}
\newcommand{\hLpcc}[1]{{ \fL_{\Psic\Psic}^{#1}}}
\newcommand{\hLpca}[1]{{ \fL_{\Psic\Psi}^{#1}}}

\newcommand{\hLs}[1]{ \fL_*^{#1}}

\startlocaldefs
\numberwithin{equation}{section}
\theoremstyle{plain}
\newtheorem{thm}{Theorem}[section]
\newtheorem{lemma}{Lemma}[section]
\newtheorem{coro}{Corollary}[section]
\newtheorem{assp}{Assumption}[]

\endlocaldefs
\begin{document}

\begin{frontmatter}
\title{Method Of Contraction-Expansion (MOCE) for Simultaneous Inference in Linear Models}
\runtitle{MOCE}

\begin{aug}
	\author{\fnms{Fei}
        \snm{Wang}\thanksref{m1}\ead[label=e1]{fwang@cargurus.com}},
        \author{\fnms{Ling}
        \snm{Zhou}\thanksref{m2}\ead[label=e2]{zhouling@swufe.edu.cn}},
\author{\fnms{Lu} \snm{Tang}\thanksref{m3}\ead[label=e3]{lutang@pitt.edu}},
\and
\author{\fnms{Peter X.-K.} \snm{Song}\thanksref{m4}\ead[label=e4]{pxsong@umich.edu}}


\affiliation{CarGurus\thanksmark{m1} and Southwestern University of Finance and
Economics\thanksmark{m2} and University of Pittsburgh \thanksmark{m3} and University of Michigan\thanksmark{m4}}

\address{CarGurus\\
  Cambridge, MA, U.S.A\\
\printead{e1}\\
}
\address{Southwestern University of Finance and Economics\\
Chengdu, China\\
\printead{e2}\\
}
\address{University of Pittsburgh\\
  Pittsburgh, PA, U.S.A\\
\printead{e3}}

\address{University of Michigan\\
  Ann Arbor, MI, U.S.A\\
\printead{e4}}
\end{aug}

\begin{abstract}
Simultaneous inference after model selection is of critical importance to
address scientific hypotheses involving a set of parameters.  In this paper, we
consider high-dimensional linear regression model in which a regularization
procedure such as LASSO is applied to yield a sparse model. To establish a
simultaneous post-model selection inference, we propose a method of contraction
and expansion (MOCE) along the line of debiasing estimation that enables us to
balance the bias-and-variance trade-off so that the super-sparsity assumption
may be relaxed. We establish key theoretical results for the proposed MOCE
procedure from which the expanded model can be selected with theoretical
guarantees and simultaneous confidence regions can be constructed by the joint
asymptotic normal distribution. In comparison with existing methods, our
proposed method exhibits stable and reliable coverage at a nominal significance
level with substantially less computational burden, and thus it is trustworthy for its
application in solving real-world problems.
\end{abstract}



\end{frontmatter}

\section{Introduction} \label{sec:intro}
We consider the linear model with a response vector $\fy=(y_1,...,y_n)^T$ and
an $n \times p$ design matrix $\fX$,
\begin{equation}
	\fy = \fX \fbetao + \fepsilon,
	\label{linear.model}
\end{equation}
where  $\fbeta^* = (\beta_1^*,\cdots,\beta_p^*)^T \in \mathbb{R}^{p}$ denotes a
$p$-dimensional vector of unknown true regression coefficients, and
$\fepsilon=(\epsilon_1,\dots,\epsilon_n)^T$ is an $n$-dimensional vector of
$i.i.d.$ random errors with mean zero and variance $\sigma^2\fI_n$, where $\fI_n$ is the
$n \times n$ identity matrix. All columns
in $\fX$ are normalized to have mean zero and $\ell_2$-norm 1. The sample covariance matrix of $p$ predictors and its corresponding population covariance matrix are denoted by $\fS =
 \frac{1}{n} \fX^T \fX$ and $\fSigma$, respectively. Let
$\mAo =\{j: \beta_j^* \neq 0, j=1,\dots,p \}$ be the support of $\fbetao$ with
cardinality $a = |\mAo|$. In this paper, assuming $p \rightarrow \infty$ as
$n\rightarrow \infty$, we focus on simultaneous statistical
inferences on a certain parameter subset of $\fbetao$ when $p \gg n$.

Arguably, in the setting of $p \gg n$, a simultaneous inference for the entire set of $p$ parameters, i.e.
$\fbetao$, is generally not tractable due to the issue of model identification.
A key assumption
widely adopted in the current literature to facilitate statistical inference is the
sparsity of $\fbetao$, namely $a\ll n$, in addition to regularity conditions on the
design matrix; see for example \cite{RSSB12094, vandegeer2014, zhangcun2014},
among others. The sparsity assumption of the true signals necessitates variable selection, which has been extensively studied in the
past two decades or so. Being one of the most celebrated variable selection
methods, Least Absolute Shrinkage and Selection Operator (LASSO)\cite[]{tibshirani94} has gained great popularity in
both theory and applications. Specifically, a LASSO estimator is obtained by
minimizing the following penalized objective function:
\begin{align}
	\hbeta_{\lambda} =\arg\underset{\fbeta \in \bR^p}{\min}\Big (\frac{1}{2n} \twonorm{\fy-\fX\fbeta }{2} + \lambda \onenorm{\fbeta}{} \Big ),
	\label{LASSO.est}
\end{align}
where $\onenorm{\cdot}{}$ is the $\ell_1$-norm of a vector and $\lambda>0$ is
the tuning parameter. Based on this LASSO estimator, $\hbeta_{\lambda}$, given in
\eqref{LASSO.est}, statistical inferences for parameters in $\fbetao$ in the
aspects of hypothesis test and confidence region construction have recently
received a surge of attention in the literature because statistical inference
has been always playing a central role in the statistical theory and providing
one of the most effective ways for the transition of data to knowledge.

Some progresses in post-model selection inferences have been reported in the
literature. The method LASSO+mLS proposed in \cite{liu2013} first performs
LASSO model selection and then draws statistical inferences based on the
selected model.  This approach requires model selection consistency and some
incoherence conditions on the design matrix
\citep{yu2006,meinshausen2009,Buhlmann2011}.  Inference procedures built upon
those conditions have been noted as being impractical and exhibited poor
performances due to the lack of uniform validity of inferential precedures over
sequences of models; see for example, \cite{leeb2008sparse,
chernozhukov2014valid}.

To overcome the reliance on  the oracle asymptotic distribution in inference,
many solutions have been proposed in recent years. Among those, three methods
are so far known for a valid post-model selection inference. (i) The first kind is sample splitting
method \cite{wasserman2009, noicolai2009, meinshausen2010stability} and
resampling method \cite{minnier2012perturbation}.
A key drawback of the sample splitting method is its requirement of the beta-min
assumption, while the resampling approach entails a strong restrictive
exchangeability condition on the design matrix. (ii) The second kind is group inference proposed in
\cite{RSSB12094}.  Unfortunately, this approach fails to show desirable power
to detect individual signals, and thus it is not useful in practical studies.
(iii) The third kind is low-dimensional projection (LDP) \cite{zhangcun2014, vandegeer2014,
java2014}. Such inferential method is rooted in a seminal idea of
debiasing, resulting from the use of penalized objective function that causes estimation shrinkage. This method will be adopted in this paper for a new paradigm of post-model selection inference. Following
 the debiasing approach proposed by \cite{zhangcun2014}, \cite{cai2017} investigates both
adaptivity and minimax rate of the debiasing estimation, which provides useful insights on the rate of model contraction and expansion considered in this paper. Specifically, an LDP estimator, $\hat \fb$, takes a debiasing step under an
operation of this form: $\hat \fb = \hbeta_{\lambda} + \frac{1}{n} \hat \Theta
\fX^T(\fy - \fX\hbeta_{\lambda})$, where $\hat \Theta$ is a sparse estimate  of
precision matrix $\Sigma^{-1}$.  When matrix $\hat\Theta$ is properly
constructed with a well-controlled behavior, the bias term, $\Delta = \sqrt n (\hat \Theta \fS -
\fI_p) (\hbeta -\fbetao)$, would become asymptotically negligible. In this case, statistical inference can be
conducted using the debiased estimator $\hat \fb$. It is known that obtaining a desirable $\hat
\Theta$ is not a trivial task due to the singularity of sample covariance
$\fS$.  For examples, \cite{vandegeer2014} proposes to use node-wise LASSO to
get $\hat \Theta$, while \cite{java2014} adopts a convex optimization algorithm
to obtain $\hat \Theta$.  It is worth noting that these existing approaches are computationally burdensome, and require extra regularity conditions to ensure the estimated sparse $\hat \Theta$ to be feasible and stable. In the
setting of the LDP estimator, \cite{zhangcheng2017} proposes a bootstrap-based
simultaneous inference for a group, {\emph {say}} $\mG$, of parameters in $\fbeta^*$ via the distribution of quantity $\max_{j\in \mG} \sqrt n |\hat b_{j} - \beta_{j}^*|$, where the bootstrap resampling,
unfortunately, demands much more computational power than a regular LDP estimator based on the node-wise
LASSO estimation $\hat \Theta$.

Overcoming the excessive computational cost on acquiring $\hat \Theta$ motivates us to consider
a ridge type of approximation to the precision matrix $\Sigma^{-1}$, in a
similar spirit to the approach proposed by \citet{Ledoit2004365} for estimation of a
high-dimensional covariance matrix.  Note that the LASSO estimator
$\hbeta_{\lambda}$ satisfies the following Karush-Kuhn-Tucker (KKT) condition:
\begin{align}
-\frac{1}{n}\fX^T\fepsilon+ \fS (\hbeta_{\lambda} -\fbetao) +  \lambda \fkappa = \fzero,
\label{norm.eq}
\end{align}
where $\fkappa = (\kappa_1, \cdots, \kappa_p)^T$ is the subdifferential of
$\onenorm{\hbeta_{\lambda}}{}$ whose $j$th component is $\kappa_j = 1$ if
$\hat\beta_{\lambda, j} > 0$, $\kappa_j = -1$ if $\hat{\beta}_{\lambda, j} <
0$, and $\kappa_j \in [-1, 1]$ if $\hat\beta_{\lambda, j} = 0$.  Let $\ftau$ be
a $p\times p$ diagonal matrix $\text{diag}(\tau_1,\cdots,\tau_p)$ with all
positive element $\tau_j>0$, $j=1,\cdots,p$. We propose to add a term $\ftau
(\hbeta_{\lambda} - \fbetao)$, and then multiply $\ifSigmatau{\ftau}$ on the both
sides of \eqref{norm.eq}, leading to an equivalent expression of \eqref{norm.eq},
\begin{equation}
  -\frac{1}{n}\ifSigmatau{\ftau} \fX^T\fepsilon + \big\{(\hbeta_{\lambda} + \lambda\ifSigmatau{\ftau} \fkappa) -\fbetao \big\} - \ifSigmatau{\ftau} \ftau (\hbeta_{\lambda} -\fbetao) =  \fzero,
  \label{kkt.stage1}
\end{equation}
where $\hbeta_{\lambda} + \lambda\ifSigmatau{\ftau} \fkappa$ is the debiasing estimator, and $\fSigmatau{\ftau} = \fS+\ftau$ is a ridge-type sample covariance matrix.
It is easy to see that on the basis of \eqref{kkt.stage1}, establishing a valid inference on $\fbetao$
becomes straightforward if $\fSigmatau{\ftau}$ is nonsingular and bias term
$\ifSigmatau{\ftau} \ftau (\hbeta_{\lambda}-\fbetao)$ may be asymptotically
negligible under a properly tuned matrix $\ftau$. The associated technical
treatments are of theoretical interest but methodologically challenging. To address such challenges,
in this paper, we propose a new approach, termed as Method of Contraction and
Expansion (MOCE).

Our solution based on the proposed MOCE offers a practically feasible way to perform a valid simultaneous post-model selection inference in which the ridge type matrix $\ftau$ is properly tuned to establish desirable theoretical guarantees. As seen later in the paper, the ridge matrix $\ftau$ plays a key role in
determining the length of confidence interval, which can vary according to signal strengths.
That is, MOCE is able to provide a wider confidence interval which is deemed for a strong signal to
achieve a proper coverage, while a shorter one for a null signal. This is because a null signal is known with zero coefficient (i.e., no need for estimation once being identified), whereas a non-null signal is only known with non-zero coefficient, which needs to be further estimated in order to construct its confidence interval, and thus incurs extra variability in inference. Specifically, MOCE takes on an expanded model $\mhA$ that is enlarged from an
initially selected model, in the hope that the bigger model may include most
of ``weak'' signals which will be handled together with strong signals in
inference. In this way, weak signals that have non-zero coefficients are separated from null signals that have zero coefficients.
Technically, we attempt to build an expanded model big enough so that it is able to cover both strong signals and most, if not all,
of weak signals under some mild regularity conditions.  Implementing the idea of model expansion
is practically feasible; for example, the LASSO method allows us not only to
identify strong signals, but also to rank predictors in a descending order via their solution paths. With a given
expanded model, MOCE modifies the original KKT condition accordingly, where
the precision matrix $\Sigma^{-1}$ is estimated by $(\fS + \ftau)^{-1}$. Under
the sparsity assumption $a=o(n/\log p)$ and some additional mild conditions,
the bias term in \eqref{kkt.stage1} vanishes asymptotically with a proper
rate, and consequently confidence region for a set of regression parameters is readily
constructed in the paradigm of MOCE.

This paper makes new contributions to the following five domains. (i) MOCE is established
under weaker sparsity conditions required for valid simultaneous inference in comparison to those given in the current literature.
That is, MOCE assumes the sparsity condition $a = o(n/\log p)$, instead of the
popular sup-sparsity assumption, $a= o( \sqrt n / \log p )$; more importantly, MOCE does not demand additional sparsity
assumptions required by the node-wise LASSO to obtain sparse estimate of the precision matrix.
(ii) MOCE is shown to achieve a smaller error bound in terms of mean squared
error (MSE) in comparison to the seminal LDP debiasing method. In effect, MOCE estimator
has the MSE rate $\|\hbetatau{\ftau} - \fbetao\|_2 = O_p(\sqrt{\tilde{a}\log(\tilde{a})/n})$
with $\tilde{a}$ being the size of the expanded model, clearly lower than
$O_p(\sqrt{ap/n})$, the rate of the LDP estimator.  (iii) MOCE enjoys both reproducibility and numerical stability in
inference because the model expansion leaves little ambiguity for
post-selection inference as opposed to many existing methods based on a selected
model that may vary substantially due to different tuning procedures
\cite{berk2013}. (iv) MOCE is advantageous for its fast computation, because of the
ridge-type regularization, which is known to be conceptually simple and computationally efficient.
It is shown that the computational complexity of MOCE is of order
$O(n(p-\tilde{a})^2)$, in comparison to the order $O(2np^2)$ of the LDP method.
(v) MOCE enables us to construct a new simultaneous test similar to the classical Wald test for a set of parameters based on
its asymptotic normal distribution.  The proposed hypothesis test method is computationally superior to the bootstrap-based test \cite{zhangcheng2017} based on the sup-norms of individual
estimation errors. All these improvements above make the MOCE method ready to be applied in real-world applications.

The rest of the paper is organized as follows. Section \ref{sec.notation}
introduces notation and Section \ref{sec.prelim} provides preliminary
results that are used in the proposed MOCE method. In Section \ref{sec.method} we discuss in detail about
MOCE and its algorithm, including computational complexity and schemes for
model expansion.  Section \ref{sec.main} concerns theoretical guarantees for
MOCE, including a new simultaneous test. Through simulation experiments,
Section \ref{sec.simulation} illustrates performances of MOCE, with comparison
to existing methods.  Section \ref{sec.discussion} contains some concluding
remarks. Some lengthy technical proofs are included in the Appendix.

\section{Notation}
\label{sec.notation}
For a vector $\fv=(\nu_1,\cdots,\nu_p)^T \in \bR^p$, the $\ell_0$-norm is $\|
\fv\|_0 = \sum_{j=}^p 1\{|\nu_j|>0\}$; the $\infty$-norm is
$\supnorm{\fv}=\underset{1\leq j\leq p}{\max}|\nu_j|$; and the $\ell_2$-norm is
$\twonorm{\fv}{2} = \sum_{j=1 }^p \nu_j^2$.  For a $p \times p$ matrix
$\fW=(w_{ij})_{1\leq i \leq j\leq p} \in \bR^{p \times p}$, the $\infty$-norm is $\maxnorm{\fW} =
\underset{1\leq j,j'\leq p}{\max} |w_{jj'}|$ and the Frobenious norm is
$\fnormt{\fW}{2} = \tr{\fW^T \fW}$ where $\tr{\fW}$ is the trace of matrix
$\fW$.  Refer to \cite{horn2012} for other matrix norms.  Let
$\rhomin{\fW}$ and $\rhomax{\fW}$ be the smallest and largest nonzero singular
values of a positive semi-definite matrix $\fW$, respectively.

With a given index subset $\mB \subset \{1,\dots,p\}$, vector
$\fv \in \bR^p$  and matrix $\fW \in \bR^{p \times p}$ can be partitioned as $\fv =
(\fv_{\mB}^T, \ \fv_{\mB^c}^T )^T$ and
$\fW= \begin{pmatrix}
		\fW_{\mB\mB} & \fW_{\mB\mB^c} \\
		\fW_{\mB^c\mB} & \fW_{\mB^c\mB^c}
	\end{pmatrix}$.
For two positive definite matrices $\fW_1$ and $\fW_2$, their L\"oewner order $\fW_1
\succ \fW_2$ indicates that $\fW_1 - \fW_2$ is  positive definite. For two sequences of real numbers
$\{u_n\}$ and $\{ v_n\}$, the expression $u_n \asymp v_n$ means that there exist positive
constants $c$ and $C$ such that $c \leq \lim \inf_n (u_n/v_n) \leq \lim \sup_n
(u_n/v_n) \leq C$.

For the self-containedness, here we introduce restricted eigenvalue $RE(s,k)$ condition and sparse
eigenvalue $SE(s)$ condition; refer to \cite{bickel2009} for more details.
For a given subset $\mJ \subset \{1,\cdots,p\}$ and a constant $k \geq 1$,
define the following subspace $\bR (\mJ, k)$ in $\bR^p$:
\[\bR (\mJ, k):= \{ \fv \in \bR^p: \onenorm{\fv_{\mJ^c}}{} \leq k
\onenorm{\fv_{\mJ}}{}\}.\]
A sample covariance matrix $\fS =
 \frac{1}{n} \fX^T \fX$ is said to satisfy the restricted eigenvalue $RE(s,k)$ condition if for $1\leq s
\leq p$ and $k>0$ there
exists a constant $\phi_0>0$  such that
\begin{align}
\min_{\substack{\mJ \subset \{1,\ldots,p\} \\
|\mJ|\leq s}}\ \min_{\substack{ \fv \in \bR(\mJ,k)}} \frac{\twonorm{ \fX \fv}{2}}{ n
\twonorm{\fv}{2}} \geq \phi_0.
\end{align}
A sample covariance matrix $\fS$ is said to satisfy the sparse
eigenvalue $SE(s)$ condition if for any $\fv \in \bR^p$ with $\| \fv \|_0 \leq
s$ it holds
\begin{align}
  0 < \lambda_{\min}(s) \leq \lambda_{\max}(s) < \infty,
\end{align}
where
\[
  \lambda_{\min}(s) := \min_{\substack{\| \fv\|_0 \leq s \\ \fv \neq \fzero}} \frac{\twonorm{\fX
  \fv}{2}}{n\twonorm{\fv}{2}}, \quad \lambda_{\max}(s) := \max_{\substack{\|\fv\|_0 \leq s \\ \fv \neq \fzero}} \frac{\twonorm{\fX
  \fv}{2}}{n\twonorm{\fv}{2}}.
\]

\section{Preliminary Results}
\label{sec.prelim}
The first regularity condition on the design matrix $\fX$ is given as follows.
\begin{assp}
  The design matrix $\fX$ in the linear model~\eqref{linear.model} satisfies the $RE(s, k)$ condition for $k=3$ and $s = a$, where $a$ is the number of non-null signals.
  \label{assp.1}
\end{assp}

Assumption~\ref{assp.1} above is routinely assumed for design matrix $\fX$ in
a high-dimensional linear model; see for example, \citep{bickel2009,
zhangcun2014}. Note that the compatibility assumption given in
\cite{vandegeer2014} is slightly weaker than Assumption
\ref{assp.1}.

As discussed above, when the bias term $\ifSigmatau{\ftau} \ftau
(\hbeta_{\lambda}-\fbetao)$ in \eqref{kkt.stage1} is asymptotically negligible,
the modified KKT \eqref{kkt.stage1} enables us to establish an asymptotic
distribution for the proposed debiasing estimator of the form:
\begin{equation}
\tilde \fbeta_{\ftau} = \hbeta_{\lambda} + \lambda \ifSigmatau{\ftau}\fkappa.
\label{betatau.est}
\end{equation}
Lemma \ref{lemma.fnormt} below assesses both Frobenious norm and $\infty$-norm
of $\ifSigmatau{\ftau} \ftau$, a key term in the bias. This lemma suggests that when $p \gg n$ it is
impossible to fully reduce the LASSO bias in $\fbeta_{\lambda} - \fbeta^*$
\cite{buhlmann2013}. Rather, in this paper, alternatively, we are able to establish an appropriate order for the ridge tuning parameters in matrix$\ftau$, with which
the resulting $\ifSigmatau{\ftau} \ftau$ is controlled at a desirable large-sample rate.
\begin{lemma}
  Consider the sample covariance $\fS = \frac{1}{n}\fX^T \fX$.
Let the ridge matrix $\ftau = \text{diag} (\tau_1,\cdots,\tau_p)$ with
$\tau_j>0$ for $j=1,\cdots,p$, and $\tau_{\min}= \min_{1\leq j\leq p}
\tau_j$ and $\tau_{\max}= \max_{1\leq j\leq p}
\tau_j$. Let $\fSigmatau{\ftau} = \fS+\ftau$. Then, the Frobenious norm and $\infty$-norm of $\ifSigmatau{\ftau} \ftau$ are given as follows, respectively:
\begin{eqnarray*}
	\max(p-n, 0)  + \frac{\min(n,
        p)}{\{\rhomax{\ftau^{-1/2}\fS\ftau^{-1/2}}+1\}^2}
	\leq \fnormt{\ifSigmatau{\ftau}  \ftau}{2}  \\
	\leq \max(p-n, 0) +  \frac{\min(n,
        p)}{\{\rhomin{\ftau^{-1/2}\fS\ftau^{-1/2}} + 1\}^2};\\
        \frac{\tau_{\min}}{\rhomax{\fS} + \tau_{\max}}  \leq \maxnorm{\ifSigmatau{\ftau} \ftau}
	\leq
        \begin{cases}
          \frac{\tau_{\max}}{\tau_{\min}}, & \text{if $p>n$}; \\
          \frac{\tau_{\max}}{\rhomin{\fS} + \tau_{\min}}, & \text{if $p\leq n$}.
	\end{cases}
\end{eqnarray*}
\label{lemma.fnormt}
\end{lemma}
Proof of Lemma \ref{lemma.fnormt} is given in Appendix \ref{proof.fnormt}. According to Lemma~\ref{lemma.fnormt},
when $p\leq n$, it is interesting to note that the $\infty$-norm $\maxnorm{\ifSigmatau{\ftau} \ftau}$ is bounded above by
$\frac{\tau_{\max}}{\rhomin{\fS} + \tau_{\min}}$. This upper bound may converge
to 0 if $\tau_{\max} = o(1)$ and $ \rhomin{\fS} = O(1)$.  On the
other hand, when $p> n$, its upper bound is $\tau_{\max}/\tau_{\min}$, which is
always greater than or equal to 1.  Hence, when $p <n$ the bias term
$\ifSigmatau{\ftau} \ftau (\hbeta_{\lambda}-\fbetao)$ can be controlled by an appropriately small $\ftau$, leading to a  simultaneous inference
on $\fbeta$
by the means of debiasing.  In contrast, the case $``p>n"$
presents the difficulty of bias reduction for $\ifSigmatau{\ftau} \ftau$. Such insight motivates us to seek for an alternative solution in the
framework of post-model selection inference, resulting in our proposed MOCE.

The proposed MOCE mimics the well-known physical phenomenon of thermal contraction and
expansion for materials with the tuning parameter $\lambda$ being an analog to
temperature.  Specifically, MOCE reduces LASSO estimation bias in two steps as
shown in Figure \ref{lab.diagram}. In the step of contraction, LASSO selects
a model $\hat \mA$, represented by the small circle in Figure \ref{lab.diagram},
which may possibly miss some signals contained in the signal set $\mA_s$. In the step
of expansion, MOCE enlarges $\hat \mA$ to form an expanded model $\mhA$, indicated
by the large circle in Figure \ref{lab.diagram}. As a result, the signal set
$\mA_s$ would be completely contained by the expanded model $\mhA$. In other words, MOCE
begins with an initial model $\hat \mA$ through the LASSO regularization which
contains most of important signals, and then expands $\hat \mA$ into a bigger
model $\mhA$ to embrace not only strong signals, but also almost all weak signals. Refer to Section \ref{lab.mha} where required
specific conditions and rules are discussed for the model expansion.
\begin{figure}
  \centering
  \includegraphics[angle=0,width=1\textwidth]{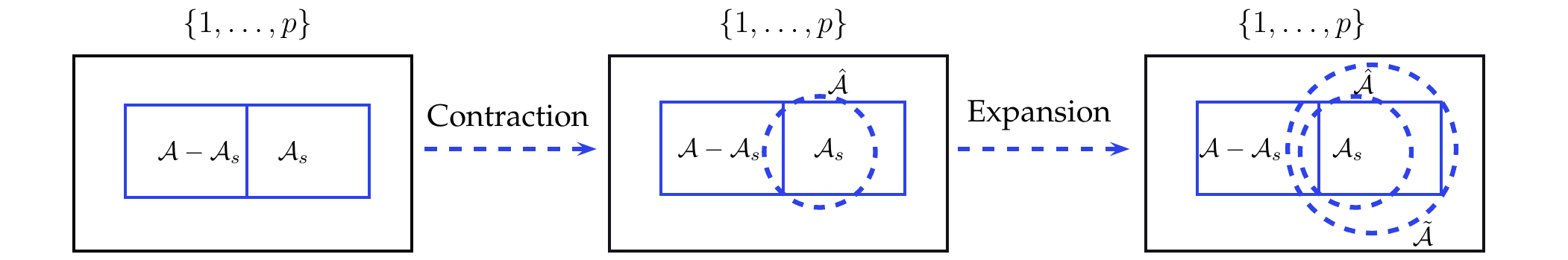}
  \caption{A schematic diagram for MOCE. The inner and outer rectangles
    respectively represent the true model $\mA$ and the full model with all $p$
    predictors  $\{1,\ldots,p\}$. The true model $\mA$ is a union of $\mA_s$
    and $\mA - \mA_s$, where $\mA_s$ denotes a signal set defined in
    \eqref{label.as}.  The small and large circles denote the LASSO selected
  model $\hat \mA$ and the expanded model $\mhA$, respectively.}
    \label{lab.diagram}
\end{figure}

We now introduce notations necessary for a further discussion on the step of model expansion. Let
$\hat \mA  =\{j: |\hat \beta_{\lambda, j}|>0,\ j=1,\cdots,p \}$ be a LASSO selected model,
whose cardinality is denoted by $\hat a = | \hat \mA|$. Here, both $\hat \mA$ and $\hat a$
are tuning parameter $\lambda$-dependent, which is suppressed for the sake of simplicity.
Similarly, let $\mhA$ be an expanded model with cardinality denoted by $\tilde a = |
\mhA|$. Given $\mA$ and $\mhA$, model expansion leads to disjoint subsets of
predictors which may conveniently be presented in Table~\ref{tab.mha} by a 2-way cross-classification, respectively, for the LASSO selected model
$\hat \mA$ (the left table) and $\hat \mA^c$ (the right table), the complement
of $\hat \mA$.
\setlength{\tabcolsep}{10pt} 
\renewcommand{\arraystretch}{1.5} 
\begin{table}[t]
 \setlength{\tabcolsep}{3pt}
 \caption{\label{tab.mha} Disjoint subsets induced by $\mA$, $\hat \mA$ and $\mhA$; the left for the LASSO selected model $\hat \mA$, and the right
 for the complementary model of $\hat \mA$. }
 \centering
  \begin{tabular}{c|c|c|ccc|c|c|c}
      & $\mhA$ & $\mhA^c$ & union & & & $\mhA$ & $\mhA^c$ & union \\
      \cline{1-4} \cline{6-9}
$\mA$ & $\mA\cap\mhA \cap \hat \mA$ & $\mA\cap\mhA^c \cap \hat \mA$ &
$\mA \cap \hat \mA$ & & $\mA$ & $\mA\cap\mhA \cap \hat \mA^c$ & $\mA\cap\mhA^c \cap \hat \mA^c$ &
$\mA \cap \hat \mA^c$ \\
$\mA^c$ & $\mA^c\cap\mhA \cap \hat \mA$ & $\mA^c\cap\mhA^c \cap \hat \mA$ &
$\mA^c \cap \hat \mA$ & & $\mA^c$ & $\mA^c\cap\mhA \cap \hat \mA^c$ & $\mA^c\cap\mhA^c \cap \hat \mA^c$ &
$\mA^c \cap \hat \mA^c$ \\
    \cline{1-4} \cline{6-9}
union & $\mhA \cap \hat \mA$ & $\mhA^c \cap \hat \mA$ &
$\hat \mA$ & & union & $\mhA \cap \hat \mA^c$ & $\mhA^c \cap \hat \mA^c$ &
$\hat \mA^c$ \\
  \end{tabular}
\end{table}
Among these subsets, two are of primary interest, namely, $\mB_{fn}$ and $\mB_{tn}$, given as follows, respectively:
\begin{align}
  \mB_{fn} = \mA \cap \mhA^c, \quad
  \mB_{tn} = \mA^c \cap \mhA^c,
	\label{notation.b}
\end{align}
and let their cardinalities are $b_{fn} = |\mB_{fn}|$
and $b_{tn} = |\mB_{tn}|$, respectively. $\mB_{fn}$ collects signals missed by
expanded model $\mhA$ (i.e., false negatives), while $\mB_{tn}$ collects all null signals that expanded
model $\mhA$ does not contain (i.e., true negatives).
With expanded model $\mhA$,
we assume that the design matrix $\fX$ satisfies Assumption \ref{assp.2}.
\begin{assp}
  The design matrix $\fX$ in the linear model~\eqref{linear.model} satisfies the sparse eigenvalue $SE(s)$ condition for $s = \max(\tilde a, \hat a)$.
  \label{assp.2}
\end{assp}
Assumption \ref{assp.2} ensures that any $s\times s$ main diagonal sub-matrices
of sample covariance matrix $S=\fX^T \fX/n$ has finite positive minimum and maximum singular values, which
essentially requires any selected model, $\hat \mA$ or $\mhA$, to have
well-defined Hessian matrices.

\section{MOCE Method}
\label{sec.method}
We first introduce MOCE and then discuss its computational
complexity. In particular, procedures for model expansion are
discussed in detail in section \ref{lab.mha}.
\subsection{MOCE}
\label{lab.alg}
Suppose an expanded model $\mhA$ has been given.  We partition a LASSO estimator
$\hbeta_{\lambda}$ given in \eqref{LASSO.est} as $\hbeta_{\lambda} =
(\hbetaat{}^T, \hbetact{}^T)^T$. Rewrite KKT condition \eqref{norm.eq}
according to this partition, respectively, for $\mhA$ and $\mhA^c$:
\begin{align}
	  \label{norm.eq.a}
	  -\frac{1}{n} \fX_{\mhA}^T(\fy- \fX_{\mhA}\hbetaat{}  -
          \fX_{\mB_{fn}}\hat \fbeta_{\mB_{fn}} - \fX_{\mB_{tn}}\hat
        \fbeta_{\mB_{tn}}) + \lambda \fkappa_{\mhA} &= \fzero, \\
	  -\frac{1}{n} \fX_{\mhA^c}^T(\fy- \fX_{\mhA}\hbetaat{}  -
          \fX_{\mhA^c}\hat \fbeta_{\mhA^c}) + \lambda \fkappa_{\mhA^c} &= \fzero.
	\label{norm.eq.ac}
\end{align}
It follows from \eqref{norm.eq.a} that
\begin{align}
  \label{lab.moce.kkt.a}
  \fS_{\mhA\mB_{fn}} (\hat \fbeta_{\mB_{fn}}  - \fbeta^*_{\mB_{fn}}) +
  \fS_{\mhA\mB_{tn}} \hat \fbeta_{\mB_{tn}} -\frac{1}{n} \fX_{\mhA}^T\fepsilon + \fSaa (\hbetaat{} - \fbetaat{}) + \lambda \fkappa_{\mhA} &= \fzero.
\end{align}
In regard to expanded model $\mhA$, the corresponding $\ftau$-matrix is an
$\tilde a \times \tilde a$ positive diagonal matrix, denoted by $\ftaua$, and the
corresponding ridge sample covariance submatrix is denoted by $\fSigmaaa =
\fSaa + \ftaua$. Adding $\ftaua(\hbetaat{}
- \fbetaat{})$ and multiplying $\ftaua$ on both sides of equation
\eqref{lab.moce.kkt.a}, we have
\begin{equation}
	\hbetaat{\ftaua}- \fbetaat{} = \frac{1}{n} \ifSigmaaa
	\fX_{\mhA}^T\fepsilon + \fr_a,
\label{norm.eq.a.2}
\end{equation}
where the debiasing estimator $\hbetaat{\ftaua}$ of subvector $\fbetaat{}$
takes the form:
\begin{equation}
\hbetaat{\ftaua} = \hbetaat{} + \lambda \ifSigmaaa \fkappa_{\mhA},
\label{est.a}
\end{equation}
and the remainder $\fr_a$ is given by
\begin{align}
  \begin{split}
  \fr_a & = \ifSigmaaa\ftaua(\hbetaat{} - \fbetaat{}) + \ifSigmaaa
  \fS_{\mhA\mB_{tn}} \hat \fbeta_{\mB_{tn}} + \ifSigmaaa
  \fS_{\mhA\mB_{fn}}(\hat \fbeta_{\mB_{fn}}  - \fbeta^*_{\mB_{fn}})\\
  & \overset{def}{=}I_{11}
  + I_{12} + I_{13}.
  \end{split}
\label{lab.ra}
\end{align}
If $\rhomax{\ftau_a} = o( \sqrt{\log p}/ n )$ holds, Lemma~\ref{lab.rcra}
shows that $\twonorm{\fr_a}{} = o_p(1/\sqrt n )$. Thus, as stated in
Theorem~\ref{thm.normal} equation \eqref{norm.eq.a.2} implies that
$\hbetaat{\ftaua}$ is consistent and follows asymptotically a normal
distribution.

Now, consider the complementary model $\mhA^c$. Following similar steps of
deriving equation \eqref{norm.eq.a.2}, we rewrite \eqref{norm.eq.ac} as
follows:
\begin{equation*}
  \fSca(\hbetaat{} - \fbetaat) +  \fSigmacc(\hbetact{} - \fbetact)  + \lambda \fkappa_{\mhA^c} \nonumber
  = \frac{1}{n} \fX_{\mhA^c}^T \fepsilon  + \ftauc(\hbetact{} - \fbetact),
\end{equation*}
where the corresponding ridge sample covariance submatrix is $\fSigmacc = \fScc
+ \ftau_c$ and $\ftau_c$ is a $(p-\tilde a) \times (p-\tilde a)$ matrix of positive diagonals. Plugging \eqref{norm.eq.a.2} and \eqref{est.a} into the
above equation, we can show
\begin{align}
  \label{lab.cc.2}
  \hbetact{\ftauc} - \fbetact
  =&\frac{1}{n} \ifSigmacc (\fX_{\mhA^c}^T
  - \fSca \ifSigmaaa \fX_{\mhA}^T)\fepsilon + \fr_c,
\end{align}
where  $\hbetact{\ftauc}$ is the debiasing estimator of subvector $\fbetact$,
which takes the following form:
\begin{equation}
\hbetact{\ftauc} = \hbetact{} + \lambda \ifSigmacc\fkappa_{\mhA^c} - \lambda \ifSigmacc \fSca \ifSigmaaa \fkappa_{\mhA},
\label{est.c}
\end{equation}
and the associated remainder term $\fr_c$ is
\begin{equation}
	\fr_c = \ifSigmacc\ftauc ( \hbetact{} -
		\fbetact)  - \ifSigmacc
                \fSca \fr_a \overset{def}{=} I_{21} + I_{22}\fr_a.
	\label{lab.rc}
\end{equation}	
If $\rhomin{\ftau_c} = O\big (\sqrt{\lambda_{\max}(p-\tilde a)} \big )$ holds, we
can show $\twonorm{\fr_c}{} = o_p(1/ \sqrt n)$ in Lemma \ref{lab.rcra}.

Now, combining the two estimators \eqref{est.a} and \eqref{est.c}, namely
$\hbetatau{\ftau} = (\hbetatau{\ftaua}^T, \hbetatau{\ftauc}^T)^T$, we express the proposed MOCE estimator for $\fbeta^*$ as follows,
\begin{align}
\hbetatau{\ftau} = \hbeta_{\lambda} + \lambda \hL^{-1}_{\ftau} \fkappa,
\label{est.all}
\end{align}
where matrix $\hL^{-1}_{\ftau}$ is a $2\times 2$ block matrix given by
\[
  \hL^{-1}_{\ftau}= 	\begin{pmatrix}
		\ifSigmaaa & \fzero \\
		-\ifSigmacc \fSca \ifSigmaaa & \ifSigmacc
		\end{pmatrix}.
\]
In comparison to equation~\eqref{betatau.est}, in~\eqref{est.all} the MOCE presents a different bias correction term, $\lambda \hL^{-1}_{\ftau} \fkappa$.
Consequently, the inverse matrix of $\hL^{-1}_{\ftau}$, $L_{\ftau}$, takes the
form of
\[
	\hL_{\ftau}= 	\begin{pmatrix}
	\fSigmaaa & \fzero \\
	\fSca  & \fSigmacc
	\end{pmatrix},
\]
which is different from the ridge covariance matrix $\hat{\Sigma}_{\ftau} = S + \ftau$ in
\eqref{betatau.est}. The fact of $\hL^{-1}_{\ftau}$ being a lower triangular
matrix implies that the MOCE estimator
$\hbetact{\ftauc}$ in \eqref{est.c} on $\mhA^c$ has no impact on
$\hbetaat{\ftaua}$ in~\eqref{est.a} on $\mhA$.

\subsection{Model expansion and size determination}
\label{lab.mha}
A primary purpose of model expansion is to control the uncertainty of model selection at a lower level than
the sampling uncertainty. This may be achieved by some regularity conditions. Intuitively,
when an expanded model $\mhA$ is too small, $\mhA$ is likely to
miss many weak signals; on the other hand, when an expanded model $\mhA$ is too
large, $\mhA$ would include many noise signals. The size of expanded model in
MOCE is critical as it pertains to a trade-off between uncertainty of model selection and efficiency of statistical inference. In
this setting, the theory for the selection of tuning
parameter $\ftaua$ and $\ftauc$ is also relevant.

\citet{Donoho94idealspatial} show that at a hard threshold $\lambda_{s+} =
\sqrt{2 \log p/n}$ LASSO can achieve the performance of an oracle within a
factor of $2\log p$ in terms of mean squared error. Under the Donoho-Johnstone's order $\lambda_{s+}$, \citet{NIPS2009_3697} develops a consistent thresholding
procedure for variable selection. For the purpose of inference, we want to have
a relatively large model to include most weak signals, so we set $\lambda_s =
\sqrt{2 \log p} /n$. We consider a factor $a^*>0$ to scale the product
$\lambda_s \sigma$, defined as the smallest integer such that
\begin{equation}
  \label{lab.lambdah}
  \sum_{i=1}^p
  \min \{ |\beta_i^*|, \lambda_s \sigma \} \leq a^* \lambda_s \sigma.
\end{equation}
Note that term $\lambda_s\sigma$ represents a compound of model selection
uncertainty $\lambda_s$ and sampling uncertainty $\sigma$. Denote a signal set
\begin{equation}
	\label{label.as}
\mA_s = \{j: |\beta_j^*|> \lambda_s \sigma, j=1,\dots,p \},
\end{equation}
whose cardinality is $a_s = |\mA_s|$. Clearly $a_s \leq a^*$. It is worth noting that
factor $a^*$ measures the overall cumulative signal strength, while size $a_s$
of $\mA_s$ is the number of signals  stronger than the corresponding factor
$\lambda_s \sigma$.  Essentially, the set given in \eqref{label.as} is formed by
the signal-to-noise ratio, where the noise arises from both model selection
uncertainty $\lambda_{s}$ and sampling error uncertainty $\sigma$. Apparently,
$\mA_s$ also contains the set of stronger signals defined by $\mA_{s+} = \{j:
|\beta_j^*|> \lambda_{s+} \sigma, j=1,\dots,p \}$.\; that is, $\mA_{s+} \subset \mA_{s}$.

For a given signal set $\mA_s$,
Assumption \ref{assp.4} below describes characteristics of expanded model $\mhA$.
\begin{assp}
  $\twonorm{\fbeta^*_{\mhA^c \cap \mA_s}}{} = o_p( 1/ \sqrt n )$.
  \label{assp.4}
\end{assp}
Assumption \ref{assp.4} is a very weak condition; first, it holds when $\tilde \mA^c \cap \mA_s
=\emptyset$, that is, expanded model $\mhA$ contains all signals. However, this full capture may be relaxed in MOCE; in other words, Assumption~\ref{assp.4} permits expanded model $\mhA$ to leak some weak signals with their
strength being order of $o_p( 1/ \sqrt n)$.

\begin{assp}
  \label{assp.3}
  $\sqrt{\hat a}\supnorm{\hat \fbeta_{\mhA^c \cap \hat \mA}}{} = o_p( 1/ \sqrt
  n )$.
\end{assp}
Assumption \ref{assp.3} is a very mild condition too, which can always be satisfied
if $\hat \mA \subseteq \mhA$. This assumption is imposed to protect rare
occasions when an initial LASSO selection ends up with a model containing
excessively many small nonzero coefficients. In this case, to proceed MOCE for
inference, Assumption \ref{assp.3} requires to choose a relatively small $\mhA$
which may not necessarily cover $\hat \mA$. As stated in Assumption~\ref{assp.4}, the leakage of very weak
signals is allowed by MOCE in inference.

When LASSO solution paths are monotonic in
$\lambda$, we may choose a hard threshold
$\lambda_a = \text{min}\{1/ \sqrt{\hat a n}, \lambda_s \}\geq 1/n$
to directly determine the size of $\mhA$. The fact of $\lambda_a$ being
smaller than $\lambda_{s+}$ implies that more variables are included in
$\mhA$. Assumption \ref{assp.3} further implies that the maximum signal strength among the false negatives and true negatives is well controlled; that is,
\begin{align}
	\label{lab.b4}
	\begin{split}
          \max \big \{ \twonorm{\hat \fbeta_{\mB_{tn}}}{}, \twonorm{\hat
          \fbeta_{\mB_{fn}}}{} \big \} & \leq \max \big \{ \twonorm{\hat
          \fbeta_{\mB_{tn}\cap \hat \mA}}{}, \twonorm{\hat \fbeta_{\mB_{fn}
        \cap \hat \mA}}{} \big \}  \\
        &\leq \sqrt {\hat a} \supnorm{\hat \fbeta_{\mhA^c \cap
        \hat \mA}}{} = o_p( 1/ \sqrt n).
	\end{split}
\end{align}

In practice, the size of $\mhA$ may be set to $\tilde a = n(1-
\lambda_a/\lambda_{\max})$ where
$\lambda_{\max}$ is the largest tuning value in LASSO solution paths at which
all parameters are shrunk to zero. We first select variables contained in $\tilde{\hat{\mA}} \overset{def}{=} \{j:
  |\hat \beta_{j,\lambda_a}|>0, j=1,\cdots,p
\} \cup \hat \mA$ into $\mhA$ if $|\tilde{\hat{\mA}}| < \tilde a$.  Next we introduce a noise injection step to
randomly select $\tilde a - \big | \tilde{\hat{\mA}} |$ predictors into $\mhA$ from variables with zero estimates at
$\lambda_s$. This noise injection step eseentially helps reduce the
sensitivity of the expanded model with variable selection relative to the
sampling variability.  It is worthy to comment that although LASSO has been the method of choice for
our procedure in this paper, in fact, the proposed MOCE allows other methods to
construct $\mhA$ as long as a chosen expanded model $\mhA$ satisfies
Assumptions \ref{assp.4} and \ref{assp.3}.    Based on above assumptions, Lemma
\ref{lab.rcra} assesses the remainder terms $\fr_a$ in \eqref{lab.ra} and $\fr_c$ in
\eqref{lab.rc} in terms of $\ell_2$-norm.

\begin{lemma}
  \label{lab.rcra}
   Suppose that Assumptions \ref{assp.1}--\ref{assp.3} hold. Assume $a = o( n/ \log p)$, $\lambda \asymp
   \sqrt{\log p/n}$ and
   \begin{align}
	   \rhomax{\ftaua} = o(\sqrt{\log p}/n), \quad \rhomin{\ftau_c} =
    O(\sqrt{\lambda_{\max}(p-\tilde a)}).
	   \label{lab.rarc.condition}
    \end{align}
    Then, $\twonorm{\fr_a}{} =o_p( 1/ \sqrt n)$ and $\twonorm{\fr_c}{} =o_p(1/ \sqrt n)$.
\end{lemma}
\begin{proof}
By the expression of $\fr_a$ in \eqref{lab.ra}, it suffices to show that three terms
$I_{11}$,  $I_{12}$ and  $I_{13}$ are all of order $o_p(1/\sqrt{n})$.
Similarly, by the expression of $\fr_c$ in \eqref{lab.rc}, the order of $\fr_c$
is established if both terms $I_{21}$ and $I_{22}\fr_a$ are all at the order of
$o_p(1/\sqrt{n})$.

For term $I_{11}$, it follows from Assumptions \ref{assp.1}-\ref{assp.2} that
\begin{align*}
\twonorm{I_{11}}{} &\leq \onenorm{I_{11}}{} \leq
\maxnorm{\ifSigmaaa\ftaua}{} \onenorm{ \hbetaat{} - \fbetaat{}}{} \leq
	O_p \Big (\rhomax{\ftaua} \sqrt{\frac{\log p}{n}} a \Big) = o_p(1/ \sqrt n),
\end{align*}
where the third inequality holds from Lemma \ref{lemma.fnormt} with $\tilde
a<n$ and $\rhomin{\fSaa}$ being bounded from below by Assumption \ref{assp.2}.
	
For term $I_{12}$, applying \eqref{lab.b4} and Assumptions \ref{assp.2} and
\ref{assp.3}, we have
\begin{align*}
  \twonorm{I_{12}}{} & = \twonorm{\ifSigmaaa \fS_{\mhA,\mB_{tn}\cap\hat \mA} \hat
  \fbeta_{ \mB_{tn} \cap \hat \mA}}{} \leq \frac{\sqrt{ \lambda_{\max}(\tilde a)
	\lambda_{\max}(\hat a )}}{\lambda_{\min}(\tilde a) + \rhomin{\ftau_a}}
  \twonorm{\hat \fbeta_{\mB_{tn} \cap \hat \mA}}{} = o_p(1/ \sqrt n).
\end{align*}

Similar to the proof of term $I_{12}$, for term $I_{13}$ we obtain
\begin{align*}
\twonorm{I_{13}}{} & \leq \frac{\sqrt{ \lambda_{\max}(\tilde a)
		\lambda_{\max}(b_{fn})}}{\lambda_{\min}(\tilde a) + \rhomin{\ftau_a}} \twonorm{\hat
	\fbeta_{\mB_{fn}} - \fbeta^*_{\mB_{fn}}}{}\\
& \leq \frac{\sqrt{ \lambda_{\max}(\tilde a)
		\lambda_{\max}(b_{fn})}}{\lambda_{\min}(\tilde a) + \rhomin{\ftau_a}} \big (\twonorm{\hat
	\fbeta_{\mB_{fn}}}{} + \twonorm{\fbeta^*_{\mB_{fn}}}{} \big )= o_p(1/ \sqrt n),
\end{align*}
where the last equality follows from \eqref{lab.b4} and $\twonorm{\fbeta^*_{\mB_{fn}}}{2} = o_p(1/n)$, which is shown below.
\begin{align*}
\twonorm{\fbeta^*_{\mB_{fn}}}{2}
\leq \twonorm{\fbeta^*_{\mhA^c \cap
		\mA_s}}{2} + (a^* - a_s) \lambda^2_h \sigma^2
\leq \twonorm{\fbeta^*_{\mhA^c \cap
		\mA_s}}{2} +  o_p(\frac{n}{\log p} \frac{\log p}{n^2}) = o_p(1/n).
\end{align*}
By the definition of $\mA_s$ we have
$\sum_{i \in \mA_s^c} (\beta_i^*)^2 +  \sum_{i \in \mA_s} \lambda_s^2 \sigma^2  \leq a^* \lambda_s^2 \sigma^2.$
Using further Assumption \ref{assp.4}, we have $\twonorm{\fr_a}{} =
o_p(1/ \sqrt n)$.

Now we turn to the assessment of $\fr_c$. For term $I_{21}$, it follows from \eqref{lab.b4} and Lemma \ref{lemma.fnormt} that
\begin{align*}
  \twonorm{I_{21}}{} &\leq \twonorm{\ifSigmacc\ftauc }{}
  \twonorm{\hbetact{} - \fbetact}{}\leq \twonorm{\hat \fbeta_{\mB_{tn} \cap \hat \mA}}{} +  \twonorm{\hat
  \fbeta_{\mB_{fn}} - \fbeta^*_{\mB_{fn}}}{} = o_p(1/ \sqrt n).
\end{align*}

For term $I_{22}\fr_a$,  under $\tau_c = O(\sqrt{\lambda_{\max}(p-\tilde a)})$, $\twonorm{\fr_a}{} =
o_p(1/ \sqrt n)$ and Assumption \ref{assp.3}, we obtain
\begin{align*}
  \twonorm{I_{22}\fr_a}{} \leq \twonorm{\ifSigmacc}{} \twonorm{\fSca}{} \twonorm{\fr_a}{} \leq
	\frac{\sqrt{\lambda_{\max}(\tilde a) \lambda_{\max}(p-\tilde a)}}{\rhomin{\ftau_c}}
  \twonorm{\fr_a}{} = o_p(1/ \sqrt n).
\end{align*}
This completes the proof for order of $\twonorm{\fr_c}{}$ being $o_p(1/ \sqrt n)$.
\end{proof}

\subsection{Computational complexity}
\label{lab.comp}
The dominant computational cost in MOCE is at calculating the inverse of $\fSigmacc$ with the computational complexity being of order $O(n(p-\tilde a)^2)$ under the operation of the the Sherman-Morrison formula. In the case where LASSO uses the
popular coordinate descent algorithm, the associated computational complexity
is of order $O(2np)$ \citep{jerome2010}, pertaining to iterations of all $p$ variables under a
fixed tuning parameter.  Debiasing methods \citep{vandegeer2014, zhangcun2014}
ought to run $p$ LASSO regressions for the node-wise LASSO, in order to obtain
a sparse estimate of the precision matrix.  Therefore, with fixed $p$ tuning
parameters, the computational complexity of the existing methods is of order
$O(2np^2)$. If computational costs on selection of tuning parameters are considered, {\emph {say}}, certain
data-driven methods such as cross-validation, arguably, the associated
computational complexity can elevate dramatically. This comparison suggests that MOCE has significantly lower computational burden than the existing
node-wise LASSO. In the implementation of MOCE, it is noted that special forms
of $\ftaua = \tau_a I$ and $\ftauc = \tau_c I$ work well, where $\tau_a$ and
$\tau_c$ are two scalars. Thus, in this case where MOCE uses only two tuning
parameters, MOCE is very appealing in real-world applications.

\section{Main results}
\label{sec.main}
In this section we present several key large-sample properties, including
asymptotic normality (ASN) under Gaussian errors and non-Gaussian errors, useful
for simultaneous inference.  In Lemma~\ref{lab.rcra}, we establish respective $\ell_2$-norm bounds for error terms $\fr_a$ and
$\fr_c$ under positive diagonal matrices $\ftau_a$
and $\ftau_c$. Because of the condition \eqref{lab.rarc.condition}, it suffices
to implement MOCE with $\ftaua = \tau_a I$ and $\ftauc = \tau_c I$, where
$\tau_a$ and $\tau_c$ are two scalars. Thus, in the remaining sections, we only
consider these special forms of $\ftaua$ and $\ftauc$.

\subsection{ASN under Gaussian errors}
\begin{assp}
	Error terms in model \eqref{linear.model}, $\epsilon_1,\dots,\epsilon_n$, are independent and identically distributed
Gaussian random variables with mean zero and variance $\sigma^2$,
$0<\sigma^2<\infty$.
  \label{lab.assp.5}
\end{assp}
We are interested in simultaneous inference in a parameter vector that contains at
most $m$ parameters where $m$ is a fixed constant smaller than $n$.  To set up the framework, we consider a $p$-dimensional
vector $\fd = (d_1,\dots,d_p)^T$ in a parameter space $\mM_{m}$ defined as follows:
\begin{equation}
\begin{split}
  \mM_{m}=\Big \{ \fd \in \bR^p : \twonorm{\fd}{}=1, \
    \| \fd\|_0 \leq m \Big \}.
\end{split}
\end{equation}

\begin{thm}
Let $\mhA$ be a size-$\tilde a $ expanded model satisfying Assumptions
\ref{assp.1}--\ref{lab.assp.5}. Let  $\fd \in \mM_{m}$, $a = o( n/ \log p)$, $ \tau_a  = o(\sqrt {\log p}/n)$, $\tau_c = O(\sqrt{
	\lambda_{\max}(p - \tilde a)})$, $v^2 = \sigma^2 \fd^T \hL^{-1}_{\ftau}\fS (\hL^{-1}_{\ftau})^T
\fd>0$, and $\lambda \asymp \sqrt{\log p/n}$. Then, the MOCE estimator
$\hbetatau{\ftau}$ in \eqref{est.all} satisfies
\begin{align*}
	\sqrt{n} v^{-1}\fd^T(\hbetatau{\ftau} - \fbetao) = \frac{1}{\sqrt n} v^{-1} d^T
        \hL^{-1}_{\ftau} \fX^T \fepsilon + o_p(1),
\end{align*}
	where $\frac{1}{\sqrt n} v^{-1} \fd^T\fL^{-1}_{\ftau} \fX^T \fepsilon$ follows $N(0, 1)$ distribution.
\label{thm.normal}
\end{thm}

\begin{proof}
	Combining \eqref{norm.eq.a.2} and \eqref{lab.cc.2} with partition $\fd= (\fd_{\mhA}^T, \fd_{\mhA^c}^T)^T$ gives
  \begin{align*}
	  \sqrt n \fd^T(\hbetatau{\ftau} - \fbetao) = \frac{1}{\sqrt n} \fd^T \fL^{-1}_{\ftau} \fX^T
    \fepsilon + \sqrt n\fd_{\mhA}^T \fr_a + \sqrt n\fd_{\mhA^c}^T \fr_c.
  \end{align*}
  Assumptions \ref{assp.1}--\ref{assp.3} imply that $\twonorm{ \sqrt n\fd_{\mhA}^T
  	\fr_a}{} = o_p(1)$ and $\twonorm{ \sqrt n\fd_{\mhA^c}^T
	\fr_c}{} = o_p(1)$ from Lemma \ref{lab.rcra}. Then, Theorem~\ref{thm.normal} follows immediately from Assumption \ref{lab.assp.5} that $\frac{1}{\sqrt n} v^{-1}\fd^T \fL^{-1}_{\ftau} \fX^T \fepsilon$ follows $N(0,1)$ distribution.
\end{proof}

Theorem \ref{thm.normal} suggests that MOCE has the following three useful
properties: (i) MOCE can perform a joint inference for transformed parameter
set specified by the space $\mM_{m}$ based on a relaxed assumption $a=o(n/\log
p)$, instead of $a=o(\sqrt n/\log p)$; (ii) MOCE avoids the ``ambiguity'' issue
of post-selection inference \citep{berk2013} caused by the instability of
selected models; (iii) as discussed in Section \ref{lab.comp}, MOCE algorithm
is much faster than existing methods using the node-wise LASSO.  Besides the three
properties, in the following sections we also show other properties for MOCE, including
(iv) smaller MSE bound than existing LDP methods; and (v) a new test for a
set of parameters, different from the bootstrap test considered by
\cite{zhangcheng2017}.

\subsection{Length of confidence interval}
Hypothetically, if we fit a data with the oracle model, the smallest variance
among the least squares estimators of nonzero parameters are bounded below by
$\sigma^2 \rhomin{\fS_{\mA\mA}^{-1}}$, while estimators of zero parameters are
zero with zero variance. Thus, the gap between the variances of respective estimators for zero
and nonzero parameters would be at least $\sigma^2 \rhomin{\fS_{\mA\mA}^{-1}}$
when the oracle model were used in analysis. This is an important property for the variances of estimators, which
should be accommodated in a valid inference. In fact, existing approaches for post-model
selection inference, including \cite{zhangcun2014,vandegeer2014,
zhangcheng2017}, have not accounted for such heterogeneity in the variances. As shown in their
simulation studies, variances of nonzero parameter estimators and variances of
zero parameter estimators are in the same order because a single tuning
process is used in the determination of tuning parameters.  This also explains why
existing methods have appeared to be more likely in reaching 95\% coverage for zero
parameters than for nonzero parameters.

The proposed MOCE estimator mitigates the above dilemma; we show that the ridge tuning
matrix with different $\ftaua$ and $\ftauc$ parameters lead to different
lengths of confidence intervals for parameters in and out expanded model
$\mhA$. Numerically, we demonstrate
that variances between estimators in $\mhA$ and $\mhA^c$ appear different in their magnitudes due to
the use of the second tuning process with the ridge matrices. In theory, Corollary
\ref{coro.2} shows that in MOCE estimation, $\hbetaat{\ftaua}$ always has a
larger variance than $\hbetact{\ftauc}$. The lower bound of
$var(\hbetaat{\ftaua})$ is at the order $O(1/ \rhomin{\fSaa})$, while the upper
bound of $var(\hbetact{\ftauc})$ is at the order $O(1/\rhomax{\fSaa})$.
Consequently, the resulting length of confidence interval differs between parameters in
$\mhA$ and $\mhA^c$.  To present Corollary \ref{coro.2}, let
$\fe_1,\dots,\fe_{\tilde a} \in \bR^p$ be the standard basis vectors that span
subspace $\bR^{\tilde a} \subset \bR^p$, and similarly let
$\fe_1^{\bot},\dots,\fe_{p- \tilde a}^{\bot} \in \bR^p$ be the standard basis
for subspace $\bR^{p-\tilde a}\subset \bR^p$.
\begin{coro}
	\label{coro.2}
	Under the same assumptions as those in Theorem \ref{thm.normal}, the minimal
	variance of $\hbetaat{\ftaua}$ is larger than the maximal variance of
	$\hbetact{\ftauc}$,
	\begin{eqnarray*}
		&&\hspace{-0.2cm}var(\hbetaat{\ftaua}) \geq \underset{1\leq i
                \leq \tilde a}{\min} \sigma^2 \fe_i^T \fL^{-1}_{\ftau}  \fS \fL^{-1}_{\ftau}
                \fe_i \geq c_1/\rhomin{\fSaa} \\
		&&\hspace{-0.2cm}\geq c_2/\rhomax{\fSaa} \geq \underset{1\leq i \leq p-\tilde a}{\max} \sigma^2(\fe_i^{\bot})^T \fL^{-1}_{\ftau}  \fS \fL^{-1}_{\ftau} \fe_i^{\bot} \geq var(\hbetact{\ftauc}),
	\end{eqnarray*}
	where $c_1$ and $c_2$ are two positive constants.
\end{coro}

Proof of Corollary \ref{coro.2} is given in Appendix~\ref{proof.coro2}.
\cite{cai2017} studied the problem about constructing an adaptive confidence
interval, in which the interval has its length automatically adjusted to the
true sparsity of the unknown regression vector, while maintaining a
pre-specified coverage probability. They showed that it is impossible to
construct a confidence interval for $\beta_i^*$ adaptive to the sparsity $a$ with
$\sqrt n / \log p \leq a \leq n/ \log p$. Our MOCE method provides valid
simultaneous inferences, and the resulting confidence interval length may, or may not be optimal, which is worth further exploration. MOCE does not attempt to construct a
confidence interval adaptive to the signal parsity as considered in \cite{cai2017}, rather adaptive to signal strengths.

\subsection{ASN under non-Gaussian errors}
When errors $\epsilon_i$'s  do not follow a Gaussian distribution, Theorem
\ref{thm.unnormal} shows that $\hbetatau{\ftau}$ still converges to a Gaussian
distribution when Assumption \ref{lab.assp.5} is replaced by Assumption
\ref{lab.assp.6}.

\begin{assp}
  \label{lab.assp.6}
  Let $w_i = \frac{1}{\sqrt n} \fd^T\fL^{-1}_{\ftau}\fx_{i}$, $\fd \in \mM_{m}$, with $\fx_i$ being the $i$th column of matrix $\fX^T =
	(\fx_{1},\cdots,\fx_{n})$. For some $r>2$,
	\begin{align*}
          \sup_{1\leq i \leq n} \bE |\epsilon_i|^r < \infty \ \text{and} \ \lim_{n\rightarrow \infty} \max_{1 \leq i \leq n} \frac{w_i^2} {\sum_{i=1}^n w_i^2} = 0.
	\end{align*}
\end{assp}

\begin{thm}
Let $\mhA$ be a size-$\tilde a $ expanded model satisfying Assumptions
\ref{assp.1}--\ref{assp.3} and \ref{lab.assp.6}. Let  $\fd \in \mM_{m}$, $a =
o( n/ \log p)$, $ \tau_a  = o(\sqrt {\log p}/n)$, $\tau_c = O(\sqrt{
\lambda_{\max}(p - \tilde a)})$, $v^2 = \sigma^2 \fd^T \hL^{-1}_{\ftau}\fS (\hL^{-1}_{\ftau})^T
\fd>0$, and $\lambda \asymp \sqrt{\log p/n}$. Then, the MOCE estimator
$\hbetatau{\ftau}$ in \eqref{est.all} satisfies
\begin{align*}
	\sqrt{n} v^{-1}\fd^T(\hbetatau{\ftau} - \fbetao) = \frac{1}{\sqrt n} v^{-1} \fd^T
        \hL^{-1}_{\ftau} \fX^T \fepsilon + o_p(1),
\end{align*}
where $\frac{1}{\sqrt n}v^{-1} \fd^T\fL^{-1}_{\ftau} \fX^T \fepsilon$ follows asymptotically $N(0, 1)$
distribution.
\label{thm.unnormal}
\end{thm}

\begin{proof}
  Following similar arguments to the proof of Theorem~\ref{thm.normal}, we have
  \begin{align*}
	  \sqrt n \fd^T(\hbetatau{\ftau} - \fbetao) &= \frac{1}{ \sqrt n} \fd^T \fL^{-1}_{\ftau} \fX^T
    \fepsilon + \sqrt n \fd_{\mhA}^T \fr_a + \sqrt n \fd_{\mhA^c}^T \fr_c\\
	  &= \frac{1}{\sqrt n}\sum_{i=1}^n w_i \epsilon_i + o_p(1).
  \end{align*}
From Assumption \ref{lab.assp.6} the Lindeberg's Condition holds because for
any $\delta>0$, as $n \rightarrow \infty$,
\begin{align*}
	  \sum_{i=1}^n \bE \Big \{\frac{w_i^2}{v^2}  \epsilon_i^2 \mathbbm{1}( \Big |\frac{w_i}{v}\epsilon_i \Big |> \delta) \Big \} &\leq \sum_{i=1}^n \bE \Big (\frac{|w_i|^r}{v^r}  |\epsilon_i|^r \frac{1}{\delta^{r-2}} \Big )\\
	  &\leq n \max_{1\leq i \leq n} \Big (\frac{|w_i|^2}{\sum_{i=1}^n w_i^2 \sigma^2} \Big )^{r/2}  \frac{\max_{1 \leq i \leq n} \bE |\epsilon_i|^r }{\delta^{r-2}} \rightarrow 0.
\end{align*}
The Lindeberg Central Limit Theorem implies that $\frac{1}{\sqrt n}
	v^{-1}\fd^T \fL^{-1}_{\ftau} \fX^T \fepsilon$ converges in distribution to $N(0,1)$.
\end{proof}

\subsection{$\ell_2$-norm error bounds}
\label{sec.ell2}
For the popular LDP method \cite{zhangcun2014}, it has been shown that the debiasing
estimator $\hbeta_{LDP}$ satisfies
\begin{equation}
	\twonorm{\hbeta_{LDP} - \fbeta^*}{} = O_p(\sqrt{a p/n} ),
\end{equation}
which is higher than $O_p(\sqrt{a \log p/ n} )$, the order that LASSO
achieves. Refer to Section 3.3 in \cite{zhangcun2014}. Below Corollary
\ref{lab.coro2} shows that MOCE's $\ell_2$-norm error bound is of order
$O_p(\sqrt{ \tilde a \log \tilde a/n})$, which is lower than
$O_p(\sqrt{a p/n} )$, the LDP's order. This improvement in the error
bound is largely resulted from the fact (i.e., Corollary~\ref{coro.2}) that MOCE controls the variances for null
signals to lower levels than those for non-null signals.  Assumption \ref{lab.assp.7} is required to establish such $\ell_2$-norm error
bound analytically.  Let $\fepsilon = (\epsilon_1, \dots, \epsilon_n)^{T}$.
\begin{assp}
	\label{lab.assp.7}
  The error $\fepsilon$ satisfies $\supnorm{\frac{1}{n}\fX_{\mhA}^T\fepsilon}{}= O_p(\sqrt{\log \tilde a /n})$ and $\supnorm{\frac{1}{n}\fX_{\mhA^c}^T\fepsilon}{} = O_p(\sqrt{\log (p-\tilde a) /n})$.
\end{assp}
This assumption is widely used in the literature of high-dimensional models, see
for examples \cite{bickel2009, negahban2012}, which can be easily verified to
be true for the case of sub-Gaussian random errors.

\begin{coro}
  \label{lab.coro2}
  Let $\mhA$ be a size-$\tilde a $ expanded model satisfying Assumptions
  \ref{assp.1}--\ref{assp.3} and \ref{lab.assp.7}. Suppose $a = o( n/ \log p)$, $\tau_a  = o(\sqrt {\log p}/n)$, and $\tau_c = O(\sqrt{
	  \lambda_{\max}(p - \tilde a)})$. Then the $\ell_2$-norm error bounds of the MOCE estimator $\hbetatau{\ftau} = (\hbetaat{\ftaua}^T, \hbetact{\ftauc}^T)^T$ in \eqref{est.all} are given by, respectively,
	\begin{align*}
		\twonorm{\hbetaat{\ftaua}- \fbetaat{}}{} &= O_p(\sqrt {\tilde a \log \tilde a/n}), \\
		\twonorm{\hbetact{\ftauc} - \fbetact}{} &= o_p\big (\max \{1/\sqrt n, \sqrt{(p - \tilde a)\log
    (p-\tilde a)/n}/ \tau_c\} \big ).
	\end{align*}
\end{coro}
The proof of Corollary \ref{lab.coro2} is given in Appendix
\ref{lab.proof.coro2}.  Note that when $\tau_c$ is chosen to be large enough,
the $\ell_2$-norm error bound of the MOCE estimator $\hbetact{\ftauc}$ will be
dominated by that of $\twonorm{\hbetaat{\ftaua}- \fbetaat{}}{}$ on the expanded
model $\mhA$, which is order $O_p(\sqrt {\tilde a \log \tilde a/n})$.

\subsection{Simultaneous test}
\label{sub.wald}
In this paper we consider a Wald-type test based on the distributional result
of Theorem \ref{thm.normal} or Theorem~\ref{thm.unnormal}.  Let $\mG$ denote a subset of $\{1,\dots,p\}$
whose cardinality $|\mG| = g$ satisfying $g/n \rightarrow \gamma\in (0,1)$.
With respect to $\mG$, $\fbetao$ and $\hbetatau{\ftau}$ can be partitioned
accordingly as $({\fbeta^*_{\mG}}^T, {\fbeta^*_{\mG^c}}^T)^T$ and $(\hbetag^T,
\hbetagc^T)^T$. We want to test the following hypothesis:
\[
	H_0:\fbetao_j=0 \ \text{for all $j \in \mG$} \ vs \ H_a:\fbetao_j \neq 0 \
	\text{for at least one $j \in \mG$}.
\]
When the number of parameters $p$ is fixed, a
natural choice of test statistic is the classical Wald
statistic, which is also known as the Hotelling's $T^2$ statistic in the multidimensional setting, given by
\begin{align}
	W_1 = n \sigma^{-2}\hbetag^T \fSigma_{\mG\mG}^{-1} \hbetag,
	\label{test.wald}
\end{align}
where $\fSigmagg=\{\fSigmalsl\}_{\mG\mG}$ and $\fSigmalsl =
\fL^{-1}_{\ftau}\fS(\fL^{-1}_{\ftau})^T$. Under the null hypothesis, as
$n \rightarrow \infty$ $W_1$ follows
asymptotically a $\chi^2$ distribution with the degree of freedoms equal to
$g$. When $g>n$, $\fSigmagg$ is singular and the Hotelling's $T^2$ test
statistic does not exist. Even when $g$ is smaller than $n$ but close to $n$,
$\fSigmagg$ is often inaccurate and unstable for the estimation of covariance matrix.
When $g/n \rightarrow \gamma \in (0,1)$, the empirical distribution of the
eigenvalues of $\fS$ spreads over the interval $[(1-\sqrt \gamma)^2, (1+\sqrt
\gamma)^2]$ \cite{bai2012spectral}.  Therefore, $\fS^{-1}$ often contains
several very large eigenvalues, so
Hotelling's $T^2$ test performs poorly, and can easily fail to control type
I error under the null hypothesis.

To construct a significance test for $\fbetag$ with a proper control of type I
error, we propose a new test statistic without involving the inverse of
$\fSigmagg$, in a similar spirit to \cite{bai1996} where a test for
the equality of mean vectors is considered in a two-sample problem. In our regression model, our
proposed test statistic $W_{bs}$ takes the follows form:
\begin{align}
	W_{bs} = \frac{n\hbetag^T \hbetag - \sigma^2 \tr{\fSigmagg}}{ \sigma
        \Big\{ 2\tr{\fSigmagg^2} \Big\}^{1/2}}.
	\label{test.newwald}
\end{align}

As stated in Theorem \ref{thm.wald} below, provided two extra assumptions,
test statistic $W_{bs}$ converges in distribution to the standard normal
distribution $N(0,1)$ under the null hypothesis. Thus, the null hypothesis is
rejected if $W_{bs}$ is greater than $100(1-\alpha)\%$ upper standard normal
percentile.
\begin{thm}
  Under the null hypothesis, suppose the same conditions in
  Theorem~\ref{thm.normal} hold. If $\fSigmagg$ converges
  to $\fSigmagg^*$ in probability and $\frac{g}{n} \rightarrow \gamma \in
  (0,1)$, then we have
  $
  W_{bs} \overset{d}{\rightarrow} N(0,1) \ \text{as $p\rightarrow \infty$ and $n\rightarrow \infty$}.
  $

  \label{thm.wald}
\end{thm}
The proof of Theorem \ref{thm.wald} is given in Appendix~\ref{proof.thm.wald}.

\section{Simulation studies}
\label{sec.simulation}
Essentially, we want to use simulations to compare our MOCE to popular LDP methods proposed by~\cite{zhangcun2014} for their performances on inference.

\subsection{Setup}
We simulate 200
datasets according to the following setup:
\[
  y = \fX\fbetao + \fepsilon, \quad \fepsilon=(\epsilon_i, \dots,
  \epsilon_n)^T, \quad \epsilon_i \overset{i.i.d.}{\sim} N(0, \sigma^2),  \quad
  i=1, \dots, n,
\]
where $\sigma = 2\sqrt{\frac{a}{n}}$, and the signal set $\mA$ is formed by a randomly sampled subset of $\{1,\cdots,p\}$ where the $a$ signal parameters are generated
from the uniform distribution $U(0.05, 0.6)$, while the rest
of null signal parameters are all set at 0. Each row of the design matrix $\fX$ is
simulated by a $p$-variate normal distribution $N(\fzero, 0.5\fR(\alpha))$,
where $\fR(\alpha)$ is a first-order autoregressive correlation matrix with
correlation parameter $\alpha$.  Each of the $p$ columns is normalized to satisfy
$\ell_2$-norm 1.

Three metrics are used to evaluate inferential performance for individual
parameters from the signal set $\mA$ and the non-signal set
$\mA^c$, separately. They include bias (Bias), coverage
probability (CP), and asymptotic standard error (ASE):
\begin{align*}
  \begin{split}
\text{Bias}_{\mA} &= \frac{1}{a} \sum_{j \in \mA}(\bE \hat \beta_j - \beta_{j}^*), \quad
	\text{Bias}_{\mA^c} = \frac{1}{p-a} \sum_{j \in \mA^c}(\bE \hat \beta_j -
        \beta_{j}^*),\\
\text{ASE}_{\mA} &= \frac{1}{a} \sum_{j \in \mA}\sqrt{\bV \text{ar}(\hat
\beta_{j})}, \quad 	\text{ASE}_{\mA^c} = \frac{1}{p-a} \sum_{j \in
\mA^c}\sqrt{\bV \text{ar}(\hat \beta_{j})},\\
        \text{CP}_{\mA}(\eta) & = \frac{1}{a} \sum_{j \in \mA} \mathbbm{1}\{\beta_{j}^* \in
        \text{CI}_j(\eta) \}, \quad 	\text{CP}_{\mA^c}(\eta) = \frac{1}{p-a} \sum_{j \in \mA^c}
        \mathbbm{1}\{0 \in \text{CI}_j(\eta) \},
  \end{split}
  \label{eq.metric}
\end{align*}
where $\bE \hat \beta_j$ is the expectation of $\hat \beta_j$, $\bV
\text{ar}(\hat \beta_j)$ is the asymptotic variance of $\hat \beta_j$, and
$\text{CI}_j(\eta)$ denotes the confidence interval for $\beta_{j}^*$ derived from
$\bV\text{ar}(\hat \beta_j)$ under the confidence level $1-\eta$, where
$\eta \in(0,1)$. The above metrics are estimated by their sample counterparts over 200 simulation replicates.

The LASSO estimator $\hbeta_{\lambda}$ is calculated by the R package \textit{glmnet}
with tuning parameter $\lambda$ selected by a 10-fold cross validation, where an estimate of the
variance parameter $\sigma^2$ is given by
\begin{equation}
	\hat \sigma^2 = \frac{1}{n - \hat a} \twonorm{\fy - \fX\hbeta_{\lambda}}{2},
	\label{eq.sigma}
\end{equation}
where $\hat a$ is the number of nonzero entries in the LASSO estimator
$\hbeta_{\lambda}$. It is shown in \cite{2013reid} that the above estimator $\hat
\sigma^2$ in \eqref{eq.sigma} is robust against changes in signal sparsity and
strength.

For MOCE, we
set $\ftau_a=\tau_a\fI$ and $\ftau_c=\tau_c\fI$ where $\tau_a=10^{-8}\sqrt{\log
p}/n$ and $\tau_c =10^{-4}\sqrt{\rho_{max}^{+}(\fSaa)
\rho_{max}^{+}(\fScc)}$, respectively.
Such difference between $\tau_a$ and $\tau_c$ is set according to Theorem
\ref{thm.normal}, reflecting the basic idea of MOCE on different tuning mechanisms with respect to
signals and non-signals. The size of expanded model $\mhA$, $\tilde{a}$, is determined by
$\lambda_a = C\min\{1/\sqrt{\hat a n}, \lambda_s \}$, where the constant
$C$ is between 4 and 12.
The competing LDP estimator proposed by \citet{zhangcun2014}, denoted
by $\hbetatau{LDP}$, is implemented by the R package \textit{hdi} with the
initial estimate obtained from the scaled LASSO.

\subsection{Inference on individual parameters}
We compare inferential performance between MOCE and LDP for 1-dimensional parameters. Consider the  following
scenarios:  $n=200$,  $p \in \{200,400,600\}$, $a=3$ and $\alpha \in \{0, 0.3\}$.

Table \ref{tab5} reports sample counterparts of Bias,
ASE, coverage probabilities for significance level 0.01 (CP99), 0.05 (CP95),
and 0.10 (CP90) over 200 rounds of simulations. First, clearly the oracle model always exhibits the best performance
among the three methods. In the oracle case, because the values of null signal parameters are known to be zero, their coverage probabilities are indeed always 1.  For the comparison between the other two
methods, Table~\ref{tab5} shows that the MOCE method outperforms
LDP method with the coverage probabilities much closer to the nominal levels regardless of correlation $\alpha =0$ or $\alpha=0.3$. Such an improvement by the MOCE method is due to the fact that MOCE uses different lengths of confidence
intervals to cover nonzero and zero parameters. It is noted that the MOCE
method has larger variances for non-null signal parameters in $\mA$ than those for null signal
parameters in $\mA^c$, confirming the theoretical result stated in Corollary \ref{coro.2}. On
the contrary, estimated variances for both signal and null signal parameters in the
LDP method are very similar.  According to \citet{vandegeer2014}, the LDP method tends
to optimize the global coverage of all parameters, making no differences between signals
or null signals, subject to the aim of achieving the overall shortest confidence
intervals for all parameters.  Reflecting to this strategy of optimality, the LDP method
typically produces standard errors for all parameters in the same order of
magnitude, and consequently the resulting standard errors for signal parameters are often
underestimated, whereas the standard errors for null signal parameters are overestimated.

\setlength{\tabcolsep}{10pt} 
\renewcommand{\arraystretch}{1.5} 
\begin{sidewaystable}[!htbp]
 \setlength{\tabcolsep}{4pt}
 \caption{ Summary statistics of Bias, ASE and coverage probability for inference in individual parameters based on Oracle, MOCE and LDP
 over 200 rounds of simulations. \label{tab5}}
 \centering
  \begin{tabular}{lccccccc|ccccc|ccccc}
\hline
&   &  & \multicolumn{5}{c}{Oracle} & \multicolumn{5}{c}{MOCE} &
\multicolumn{5}{c}{LDP}\\ \cline{4-18}
 $\alpha$ & $p$ &  & Bias & ASE & CP99 & CP95 & CP90 & Bias & ASE & CP99 & CP95 & CP90 & Bias & ASE & CP99 & CP95 & CP90 \\
\hline
\multirow{6}{*}{0} &
\multirow{2}{*}{200 } &
 $\mA$   &  0.000& 0.017& 0.982& 0.935& 0.882&  0.001& 0.033& 0.990& 0.940& 0.890& -0.002& 0.017& 0.957& 0.908& 0.852\\
& &
 $\mA^c$ &  0.000& 0.000& 1.000& 1.000& 1.000&  0.000& 0.026& 0.989& 0.950& 0.899&  0.000& 0.017& 0.970& 0.931& 0.883\\
&
\multirow{2}{*}{400 } &
 $\mA$   &  0.000& 0.017& 0.995& 0.953& 0.887& -0.001& 0.035& 0.987& 0.950& 0.888& -0.001& 0.018& 0.983& 0.947& 0.882\\
& &
 $\mA^c$ &  0.000& 0.000& 1.000& 1.000& 1.000&  0.000& 0.022& 0.991& 0.952& 0.902&  0.000& 0.018& 0.986& 0.945& 0.895\\
&
\multirow{2}{*}{600 } &
 $\mA$   & -0.001& 0.017& 0.983& 0.945& 0.903& -0.004& 0.054& 0.975& 0.937& 0.873& -0.003& 0.018& 0.982& 0.912& 0.865\\
& &
 $\mA^c$ &  0.000& 0.000& 1.000& 1.000& 1.000&  0.000& 0.018& 0.991& 0.953& 0.904&  0.000& 0.018& 0.990& 0.950& 0.899\\
\hline
\multirow{6}{*}{0.3} &
\multirow{2}{*}{200 } &
 $\mA$   &  0.000& 0.017& 0.987& 0.947& 0.910&  0.001& 0.041& 0.988& 0.943& 0.883& -0.002& 0.018& 0.952& 0.913& 0.857\\
& &
$\mA^c$  &  0.000& 0.000& 1.000& 1.000& 1.000&  0.000& 0.034& 0.991& 0.950& 0.902&  0.000& 0.018& 0.950& 0.913& 0.863\\
&
\multirow{2}{*}{400 } &
 $\mA$   &  0.000& 0.017& 0.985& 0.942& 0.878& -0.002& 0.036& 0.985& 0.953& 0.895& -0.002& 0.018& 0.973& 0.918& 0.865\\
& &
 $\mA^c$ &  0.000& 0.000& 1.000& 1.000& 1.000&  0.000& 0.023& 0.991& 0.953& 0.903&  0.000& 0.018& 0.985& 0.946& 0.896\\
&
\multirow{2}{*}{600 } &
 $\mA$  & -0.001& 0.017& 0.988& 0.953& 0.893& -0.004& 0.056& 0.987& 0.952& 0.910& -0.003& 0.018& 0.985& 0.943& 0.895 \\
& &
 $\mA^c$&  0.000& 0.000& 1.000& 1.000& 1.000&  0.000& 0.019& 0.991& 0.953& 0.905&  0.000& 0.018& 0.990& 0.949& 0.899 \\
\hline

  \end{tabular}
\end{sidewaystable}

Another difference between MOCE and LDP methods is computational efficiency.
Table \ref{tab5.1} reports the average computation time in one randomly selected replicate. It is evident that the MOCE method is several hundred times faster than the LDP method in all six scenarios considered in the simulation study. This is the numerical evidence confirming the theoretical computational complexity discussed in Section \ref{lab.comp}; the computational complexity for MOCE and LDP are $O(n(p-\tilde
a)^2)$ and $O(2np^2)$, respectively, for $p$ fixed
tuning parameters in the node-wise LASSO. In practice, the node-wise LASSO needs to be
calculated along a solution path with varying tuning parameters, which, with no doubt, will
dramatically increase LDP's computational cost.

 \setlength{\tabcolsep}{10pt} 
\renewcommand{\arraystretch}{1.5} 
\begin{table}[!htbp]
 \setlength{\tabcolsep}{4pt}
 \caption{Average computation time in one simulated dataset for MOCE and LDP methods. \label{tab5.1} }
 \centering
  \begin{tabular}{cccc}
\hline
&     & \multicolumn{2}{c}{Computation Time (seconds)} \\
$\alpha$ & $p$ &    MOCE & LDP\\
\hline
\multirow{3}{*}{0}
& 200 & 0.313& 228.362\\
& 400 & 0.636& 260.250\\
& 600 & 1.393& 418.376\\
\hline
\multirow{3}{*}{0.3}
& 200 & 0.345& 224.270\\
& 400 & 0.642& 249.809\\
& 600 & 1.286& 357.721\\
\hline
  \end{tabular}
\end{table}

\subsection{Simultaneous Test for a group of parameters}
\label{sub.sign}
In this second simulation study we assess the performance of \citet{bai1996}'s
test $W_{bs}$ defined in \eqref{test.newwald} for a group of parameters in comparison
to the classical Wald statistic $W_1$ given in \eqref{test.wald}. Under the same
setting of the above simulation study,
we consider a hypothesis $H_0: \fbetag = \fzero \ \text{vs} \ H_a: \fbetag \neq
\fzero$, where the size of $\mG$ is set at 5, 50 and 100.  We also consider varying different size of intersection $\mG \cap
\mAo$. When $|\mG \cap \mAo|=0$, the null hypothesis $H_0$ is true; otherwise
the alternative hypothesis $H_a$ is the case.

Empirical type I errors and power are computed under the significance level
0.05 over 200 replications. Since the asymptotic distribution of the Wald statistic \eqref{test.wald}
is constructed under the assumption that $\frac{g}{n}\rightarrow 0$ as $n
\rightarrow \infty$, we expect that $W_1$ would work well for the low-dimensional case
$|\mG|=5$ when $p$ is not too large but fails to control type I errors
when either $|\mG|$ or $p$ is large.

Table \ref{tab3} summarizes both empirical type I errors and power of $W_1$ and $W_{bs}$
based on 200 replications, where $|\mG\cap \mAo|=0$ and $|\mG\cap \mAo|>0$
correspond to type I error and power, respectively. When $|\mG|=5$ and $p = 200, 400$, the
Wald statistic $W_1$ is able to reasonably control the type I error, and appears to have
comparable power to $W_{bs}$. When $|\mG|=50, 100$ and $p = 600$, $W_1$ fails to control type
I errors properly that are much lower than 0.05 level. This implies that $W_1$ is too conservative
for simultaneous inference in high-dimensional setting.  In contrast, the proposed statistic $W_{bs}$ has clearly demonstrated
proper control of type I error and satisfactory power in all these cases.

\setlength{\tabcolsep}{10pt} 
\renewcommand{\arraystretch}{1.5} 
\begin{table}[!htbp]
 \setlength{\tabcolsep}{4pt}
 \caption{\label{tab3} Empirical type I error and power of the classical Wald statistics
	 $W_1$ and the proposed $W_{bs}$ over 200 replications under AR-1
       correlated predictors with correlation $\alpha=0$ and $\alpha=0.3$. }
 \centering
  \begin{tabular}{ccccccccc}
  \hline
  & & &  \multicolumn{2}{c}{$p=200$} &  \multicolumn{2}{c}{$p=400$} & \multicolumn{2}{c}{$p=600$} \\
\hline
 $\alpha$ &$|\mG|$ & $|\mG\cap \mAo|$  & $W_{bs}$ & $W_{1}$ & $W_{bs}$ & $W_{1}$ & $W_{bs}$ & $W_{1}$ \\
\hline
\multirow{9}{*}{0} &
\multirow{3}{*}{5}
  &0 &0.045& 0.045& 0.050& 0.025& 0.040& 0.035\\
& &2 &1.000& 1.000& 1.000& 1.000& 1.000& 1.000\\
& &3 &1.000& 1.000& 1.000& 1.000& 1.000& 1.000\\
\cline{2-9}
&
\multirow{4}{*}{50}
  &0 &0.035& 0.045& 0.050& 0.045& 0.040& 0.005\\
& &2 &0.990& 1.000& 1.000& 1.000& 0.975& 0.970\\
& &3 &1.000& 1.000& 1.000& 1.000& 1.000& 1.000\\
\cline{2-9}
&
\multirow{4}{*}{100}
  &0 &0.040& 0.020& 0.065& 0.020& 0.030& 0.000\\
& &2 &0.920& 1.000& 0.980& 0.995& 0.800& 0.565\\
& &3 &1.000& 1.000& 1.000& 1.000& 1.000& 0.995\\
\hline
\multirow{9}{*}{0.3} &
\multirow{3}{*}{5}
  &0 &0.055& 0.050& 0.045& 0.050& 0.055& 0.040\\
& &2 &1.000& 1.000& 1.000& 1.000& 0.980& 0.980\\
& &3 &1.000& 1.000& 1.000& 1.000& 1.000& 1.000\\
\cline{2-9}
&
\multirow{4}{*}{50}
  &0 &0.060& 0.075& 0.040& 0.055& 0.050& 0.010\\
& &2 &0.810& 1.000& 1.000& 1.000& 0.700& 0.645\\
& &3 &1.000& 1.000& 1.000& 1.000& 1.000& 1.000\\
\cline{2-9}
&
\multirow{4}{*}{100}
  &0 &0.055& 0.050& 0.050& 0.010& 0.040& 0.000\\
& &2 &0.520& 1.000& 0.955& 0.995& 0.405& 0.275\\
& &3 &0.970& 1.000& 1.000& 1.000& 0.930& 0.980\\
\hline

  \end{tabular}
\end{table}

\section{Discussion}
\label{sec.discussion}
We developed a new method of contraction and expansion (MOCE) for simultaneous
inference in the high-dimensional linear models. Different from the existing
low dimensional projection (LDP) method, in MOCE we propose a step of model
expansion with a proper expansion order, so that the model
selection uncertainty due to the LASSO tuning parameter is well controlled and
asymptotically ignorable in comparison to the sampling uncertainty. It is
notoriously hard to quantify model selection uncertainty in the regularized estimation procedure with variable selection. The proposed step of model
expansion overcomes this difficulty; instead of quantifying it analytically, our MOCE method controls and reduces
it asymptotically in comparison to the level of sampling uncertainty. Thus, the MOCE method provides a realistic solution to valid
simultaneous post-model selection inferences. We have thoroughly discussed the issue of determining
the size of expanded model and established as a
series of theorems to guarantee the validity of the MOCE method.
We showed both analytically and numerically that the MOCE method gives better
control of type I error and much faster computation than the existing LDP method. In addition,
a new test $W_{bs}$ provides an appealing approach to a simultaneous test for a group of parameters, with
a much better performance than the classical Wald test.

Another useful technique in the MOCE pertains to a ridge-type
shrinkage, which is imposed not only to enjoy computational speed but also to incorporate
different lengths of confidence intervals for signal and null signal parameters. It is worth noting that our MOCE method attempts to provide an adaptive construction of confidence interval to signal strength, instead of signal sparsity as proposed by \cite{cai2017}. The optimality studied in \cite{cai2017} might offer an opportunity to develop a desirable tuning procedure for the ridge $\ftau$-matrix, which is certainly an interesting future research direction. In this paper, we focus on the study of asymptotic orders of tuning parameters, where we
propose a tuning parameter selection rate
$\sqrt{2\log p}/n$ for the selection of expanded model. In effect, as suggested in our
theoretical work, asymptotical normality can be established at a rate of $\sqrt{2\log p}/n^{1+\delta}$ for any $\delta \geq 0$. Thus, we conjecture that $\sqrt{2\log p}/n$ is the
lower bound of the legitimate rate for a proper expanded model. In other
words, a rate lower than $\sqrt{2\log p}/n$ would hamper the model
selection uncertainty from being asymptotically ignorable with respect to the sampling
uncertainty. This is an important theoretical question worth further
exploration. As suggested by one of the reviewers, it is also interesting to use the magnitude of $\kappa_j, j = 1, \dots, p$ in the KKT condition to determine an expanded model, which is worthy further exploration.

An interesting direction of research on MOCE is to understand its potential
connection to elastic-net \citep{zou05}. Because both MOCE and elastic-net
perform a combined regularization via $\ell_1$-norm and $\ell_2$-norm,
there might exist a certain connection between these two approaches;
unveiling such relationship may points to a new direction of future research.

In summary, the new key contributions of MOCE that make the method useful in real-world applications include (i) confidence interval constructed by  MOCE has different lengths
for signal and null-signal parameters, and consequently MOCE can satisfactorily control
type I error; and (ii) MOCE enjoys fast computation and
scalability under less stringent regularity conditions.  Note that MOCE only involves two additional tuning
parameters $\tau_a$ and $\tau_c$ in a ridge-type regularization, while existing methods such as LDP,
bootstrap sampling and sample splitting method all involve substantially
computational costs.

\section*{Acknowledgements}
We thank the editor, the associate editor, two anonymous referees for comments that have led to significant improvements of the manuscript. This research is supported by a National Institutes of Health grant R01ES024732
and a National Science Foundation grant DMS1513595. We are grateful to Dr.
Cun-Hui Zhang for providing an R code for the low-dimensional projection (LDP)
method.  Drs. F. Wang and L. Zhou are the co-first author of this paper.

\renewcommand{\thethm}{{\sc A.\arabic{theorem}}}
\renewcommand{\thelemma}{A.\arabic{lemma}}\setcounter{lemma}{0}
\renewcommand{\thesection}{{A.\arabic{section}}}

\renewcommand{\theequation}{{A.\arabic{equation}}}
\renewcommand{\thesubsection}{{\it A.\arabic{subsection}}}
\setcounter{equation}{0}
\setcounter{subsection}{0}

\section*{Appendices}

\subsection{Proof of Lemma \ref{lemma.fnormt}}
\label{proof.fnormt}
\begin{proof}

Let $\fS = \fU \fD \fU^T$ be the singular value decomposition of $\fS$,
whose singular values are arranged in
$\fD = \text{diag}\{\rho_{(1)}, \dots, \rho_{(m)}, 0, \dots, 0\}$ with $\rho_{(1)} \geq \dots \geq \rho_{(m)} > 0 = \rho_{(m+1)} = \dots = \rho_{(p)}$. Let  $\ftau^{-1/2}\fS\ftau^{-1/2} = \fU_1\fD\fU_1^T$ be the singular value decomposition of $\ftau^{-1/2}\fS\ftau^{-1/2}$. 
Denote $\fU = \ftau^{1/2} \fU_1$. Then we have $\ftau = \fU\fU^{T}$ and $\fS = \fU\fD\fU^{T}$. By some simple calculations we obtain
\begin{align*}
\fnormt{\ifSigmatau{\ftau}\ftau}{2} =
		&\text{tr}\left\{(D + I)^{-1}(\fU^{T} \fU)^{-1} (D + I)^{-1} \fU^{T}\fU \right\}\\
		=&\sum_{j=1}^p \frac{1}{(\rho_{(j)}+1)^2}
		\leq \max(p-n, 0) + \frac{\min(n, p)}{(\rho_{(m)} + 1)^2},
	\end{align*}
	where the second equality holds due to the equation $[(\fU^{T} \fU)^{-1} (D + I)^{-1} \fU^{T}\fU]_{jj} = \frac{1}{\rho_{(j)}+1}$. Here $[A]_{jj}$ denotes the $j$th diagonal element of matrix $A$. Likewise,
	\begin{align*}
          \fnormt{\ifSigmatau{\ftau} \ftau}{2}
	  &\geq 
	  \max(p-n, 0) + \frac{\min(n, p)}{(\rho_{(1)} + 1)^2}.
	\end{align*}	
        By combining the above two inequalities, the first inequality with the
        Frobenius norm of part (ii) follows. Now we turn to the proof of the
        second inequality.
	By Theorem 4.3.1 in \cite{horn2012}, we know
	\[
          \xi + \rhomin{\ftau} \leq \rhomin{\fSigmatau{\ftau}} \leq
          \rhomax{\fSigmatau{\ftau}} \leq \rhomax{\fS} + \rhomax{\ftau},
	 \]
         where $\xi=0$ if $p>n$ and $\xi = \rhomin{\fS}$ if $p\leq n$. It
         follows immediately that
	\[
          \frac{1}{\rhomax{\fS} + \rhomax{\ftau}} \leq \rhomin{\ifSigmatau{\ftau}} \leq \rhomax{\ifSigmatau{\ftau}} \leq
	  \frac{1}{\xi + \rhomin{\ftau}}.
	\]
	Since $\ifSigmatau{\ftau}$ is positive definite, the largest element
	of $\ifSigmatau{\ftau}$ always occurs on its main diagonal, equal to
	$\maxnorm{\ifSigmatau{\ftau}} =\underset{1\leq i \leq p}{\max}\fe_i^T
	\ifSigmatau{\ftau}\fe_i$, which satisfies
	\[
          \frac{1}{\rhomax{\fS} + \rhomax{\ftau}} \leq \underset{1\leq j \leq
          p}{\max}\fe_j^T \ifSigmatau{\ftau}\fe_j \leq \frac{1}{\xi +
          \rhomin{\ftau}},
	\]
        where $\fe_1,\dots,\fe_p$ are the standard basis of Euclidean $\bR^p$
        space.  Because diagonal matrix $\ftau \succ 0$ (positive-definite),
        \[
          \maxnorm{\ifSigmatau{\ftau}\ftau} \leq \maxnorm{\ifSigmatau{\ftau}}\maxnorm{\ftau}
	  \leq \frac{\rhomax{\ftau}}{\xi + \rhomin{\ftau}} =
	  \begin{cases}
            \frac{\rhomax{\ftau}}{\rhomin{\ftau}}, & \text{if $p>n$}; \\
            \frac{\rhomax{\ftau}}{\rhomin{\fS} + \rhomin{\ftau}}, & \text{if $p\leq n$},
	  \end{cases}
        \]
	and
        \[
          \maxnorm{\ifSigmatau{\ftau}\ftau}
	  \geq \frac{\rhomin{\ftau}}{\rhomax{\fS} + \rhomax{\ftau}}.
        \]
	Then the inequality in part (ii) for the $\infty$-norm follows.
\end{proof}

\iffalse
\subsection{Proof of Lemma \ref{lemma.ac}}
\label{proof.lemma.ac}
\begin{proof}
	First, we show \eqref{ineq.m12}. With $\tilde a<n$, the $\hLpaa{}$ is a matrix
	of rank $\tilde a$ containing a nonsingular $\tilde a \times \tilde a$
	submatrix of $\fS +\tau_a\fI$. Then, the interlacing inequality
	\citep{thompson19721} implies \eqref{lemma.a4}, that is
	\[
		\rhomin{\fSpsiaa } + \tau_a \leq \rhomin{\hLpaa{}}\leq \rhomax{\hLpaa{}} \leq
	\rhomax{\fSpsiaa } + \tau_a.
	\]
	
	We now show \eqref{ineq.v.lower}. With similar arguments, it follows
	that $\rhomax{\Psic
		(\fS  + \tau_c \fI)\Psic}\leq \rhomax{\fSpsicc} + \tau_c$ by the
	interlacing inequality.  If $p -\tilde a > n$, then the $(p -\tilde a) \times (p
	-\tilde a)$ submatrix of $\fS$ contained in $\hLpcc{}$ is singular. Therefore,
	$\rhomin{\Psic (\fS  + \tau_c \fI)\Psic}\geq \tau_c$ in \eqref{lemma.a5} follows.
	
	To verify \eqref{ineq.m12}, we calculate term $(\fapa)^TM_{12}\fapc$ as follows:
	\begin{equation*}
	\begin{split}
	&(\fapa)^TM_{12}\fapc \\
	=& (\fapa)^T(\hLpaa{-1}\fSpsiac \hLpcc{-1} - \hLpaa{-1} \fSpsiaa \hLpaa{-1} \fSpsiac \hLpcc{-1})\fapc \\
	=& (\fapa)^T(\hLpaa{-1} - \hLpaa{-1} \fSpsiaa \hLpaa{-1}) \fSpsiac \hLpcc{-1}\fapc\\
	=& \tau_a (\fapa)^T\hLpaa{-1} \hLpaa{-1} \fSpsiac \hLpcc{-1}\fapc \\
	\leq& \tau_a \twonorm{(\fapa)^T\hLpaa{-1} \hLpaa{-1} }{} \twonorm{\fSpsiac \hLpcc{-1}\fapc}{} \\
	\leq& \tau_a \rhomin{\fSpsiaa} ^{-2} \rhomax{\fSpsica\fSpsiac}  \tau_c^{-2}\\
	\leq& \tau_a b_0^{-2}  \tau_c^{-2}\rhomax{\fSpsica}^2,
	\end{split}
	\end{equation*}
	where the second inequality is the result of applying \eqref{lemma.a4} and
	\eqref{lemma.a5}, and the last one is the result of assumption \ref{con.aa1}.
	Inequality \eqref{ineq.v.lower} can be easily verified by assumption
	\ref{con.aa1} and $\fa \in \mM_{c_0}$.
	
	To show \eqref{ineq.v}, notice that
	\begin{align*}
	&d^T \fL^{-1}\fS(\fL^{-1})^T d \\
	=& (\fapa)^TM_{11}\fapa + (\fapc)^TM_{22}\fapc + 2(\fapa)^T M_{12} \fapc\\
	\leq& (\fapa)^T \hLpaa{-1} \fSpsiaa \hLpaa{-1}\fapa + (\fapc)^T \hLpcc{-1} \fSpsicc \hLpcc{-1}\fapc \\
	&+2  \tau_a \tau_c^{-2} (\rhomin{\fSpsiaa} + \tau_a)^{-2} \rhomax{\fSpsica\fSpsiac}\\
	\leq& \rhomin{\fSpsiaa} ^{-2}\rhomax{\fSpsiaa}  + \tau_c^{-2}\rhomax{\fSpsicc} \\
	&+2  \tau_a \tau_c^{-2} \rhomin{\fSpsiaa} ^{-2} \rhomax{\fSpsica\fSpsiac}\\
	\leq& b_0^{-2}b_1 + \tau_c^{-2}\rhomax{\fSpsicc} + 2 \tau_a \tau_c^{-2}b_0^{-2}\rhomax{\fSpsica}^2.
	\end{align*}
	Finally, we show \eqref{lemma.a3}. Using the partition defined in \eqref{notation.b}, we obtain $ \twonorm{
		\Psic\fbetao}{} = \twonorm{\fbetaobb}{} \leq (\supnorm{\fbetaobb})
                |\mB_{fn}|^{1/2}$. Assumption \ref{con.c1} implies the discrete random variable
                $|\mB_{fn}|$ converges to 0 in probability. Thus, \eqref{lemma.a3} follows because
	for any $\delta>0$,
	\begin{align*}
	\bP\big(\twonorm{ \Psic\fbetao}{} >n^{-\gamma} \delta \big ) &\leq
        \bP\big(\supnorm{\fbetaobb} |\mB_{fn}|^{1/2} > n^{-\gamma} \delta \big)
        \leq \bP(|\mB_{fn}|\geq 1) \rightarrow 0.
	\end{align*}
\end{proof}

\subsection{Proof of Theorem \ref{theorem.1}}
\label{proof.theorem.1}
\begin{proof}
	Let $\fX^T = (\fx_{1+},\dots,\fx_{n+})$ and $w_i =
	n^{-1/2}v^{-1}\ifSigmatau{\ftau}\fx_{i+}$. Then,
	\begin{align}
	n^{1/2} v^{-1}\fa^T(\hbetatau{\ftau} - \fbetao) = T_1 + T_2,
	\label{decompose.12}
	\end{align}
	where $T_1= n^{1/2} v^{-1}\fa^T \ifSigmatau{\ftau} \ftau (\hbeta - \fbetao)$
	and  $T_2 = \sum_{i=1}^n\fa^T w_i \epsilon_i$. It is easy to show by
	condition \ref{con.a0} that $T_2 \sim N(0,1)$ and
	\begin{equation}
	\label{t2.cov}
	\text{cov}(T_2) = v^{-2} \sigma^2 \fa^T \ifSigmatau{\ftau}(\frac{1}{n} \sum_{i=1}^n \fx_{i+} \fx_{i+}^T) \ifSigmatau{\ftau} \fa = 1.
	\end{equation}
	For $T_1$, we obtain
	\begin{align}
		\label{eq:v}
		v^2 \geq \sigma^2\{\rhomax{\ifSigmatau{\ftau}}\}^{-2}\rhomin{\fS} \geq \sigma^2 b_1^{-2}b_0>0.
	\end{align}
Substituting the inequality (\ref{eq:v}) into $T_1$, we obtain
\begin{align*}
		|T_1| & \leq n^{1/2} v^{-1}\onenorm{\fa}{} \supnorm{ \ifSigmatau{\ftau} \ftau(\hbeta - \fbetao)}{} \\
		&\leq \sigma^{-1} n^{1/2}p^{1/2}c^{-1}(\rhomax{\fS} + \tau)\maxnorm{\ifSigmatau{\ftau} \ftau} \onenorm{\hbeta - \fbetao}{} \\
		& \leq 4\phi_0^{-2}\sigma^{-1} n^{1/2}p^{1/2}c^{-1}\frac{\rhomax{\fS} + \tau}{\rhomin{\fS} + \tau}\tau a\lambda\\
		& \leq 4\phi_0^{-2}\sigma^{-1} n^{1/2}p^{1/2}c^{-1}\frac{\rhomax{\fS} + \tau}{c + \tau}\tau a\lambda\\
		& =O\left(n^{1/2}p^{1/2}\rhomax{\fS}\tau\right)o_p(1)= o_p(1),
\end{align*}
where the third inequality comes from Theorem 6.1 in \cite{Buhlmann2011} and
the result (ii) of Lemma~\ref{lemma.fnormt} by using condition \ref{con.a3}.
Moreover, the last equation follows from conditions \ref{con.a2} and
\ref{con.a3}, $\tau = O\left(n^{-1/2} p^{-1/2}\rhomax{\fS}^{-1}\right)$ and
$\lambda = O(\log p/n)^{1/2}$. Thus, Theorem~\ref{theorem.1} follows.
	
\end{proof}

\subsection{Proof of Theorem \ref{thm.normal}}
\label{proof.thm.normal}
\begin{proof}
	We decompose $n^{1/2} v^{-1}\fa^T(\hbetatau{\ftau} - \fbetao)$ into five components
	\begin{align*}
	&n^{1/2} v^{-1}\fa^T(\hbetatau{\ftau} - \fbetao) \\
	=& n^{1/2} v^{-1}\fa^T(\fI-\hL^{-1}\fS)(\hbeta-\fbetao) + n^{1/2} v^{-1}\fa^T\frac{1}{n}\hL^{-1}\fX^T \fepsilon \\
	=&n^{1/2} v^{-1}\fa^T(\Psi + \Psic-\hLpaa{-}\fS - \hLpcc{-}\fS + \hLpcc{-}\hLpca{}\hLpaa{-}\fS)(\hbeta-\fbetao) \\
	&+ n^{-1/2} v^{-1}\fa^T \hL^{-1}\fX^T \fepsilon \\
	\overset{def}{=} & T_{1} + T_{2}+T_{3} + T_{4}  + T_{5},
	\end{align*}
	where
	\begin{align*}
	T_{1} =& n^{1/2} v^{-1}\fa^T(\Psi -\hLpaa{-}\fS ) (\Psi\hbeta-\Psi\fbetao), \\
	T_{2} =& n^{1/2} v^{-1}\fa^T(\Psi -\hLpaa{-}\fS ) (\Psic\hbeta-\Psic\fbetao), \\
	T_{3} =& n^{1/2} v^{-1}\fa^T(\Psic - \hLpcc{-}\fS + \hLpcc{-}\hLpca{}\hLpaa{-}\fS)(\Psi\hbeta-\Psi\fbetao), \\
	T_{4} =& n^{1/2} v^{-1}\fa^T(\Psic - \hLpcc{-}\fS + \hLpcc{-}\hLpca{}\hLpaa{-}\fS)(\Psic\hbeta-\Psic\fbetao),\\
	T_{5} =& n^{-1/2} v^{-1}\fa^T\hL^{-1}\fX^T \fepsilon.
	\end{align*}
	First, Lemma \ref{lemma.ac} implies $v = O(1)$.
	By the definition of $\Psi$ and $\mhA$, it is easy to know
	$\Psic\hbeta = \fzero$ and $\Psi \Psic = \fzero$.  And \eqref{lemma.a4} implies
	$\twonorm{\hLpaa{-}}{} \leq (\rhomin{\Psi\fS\Psi} + \tau_a)^{-1}<b_0^{-1}$. When $\tilde a\leq n$  we know $\maxnorm{\Psi
		-\hLpaa{-} \Psi\fS\Psi}=O(\tau_a)$ according to
	Lemma \ref{lemma.fnormt}. Thus, applying the same arguments in Theorem
	\ref{theorem.12}, we obtain
	\begin{align*}
	|T_{1}| &= n^{1/2} v^{-1}|\fa^T(\Psi -\hLpaa{-}\fS ) (\Psi\hbeta-\Psi\fbetao)|\\
	&=n^{1/2} v^{-1}|(\fa^T\Psi + \fa^T \Psic)(\Psi -\hLpaa{-}\Psi\fS \Psi - \hLpaa{-}\Psi \fS \Psic)(\Psi\hbeta-\Psi\fbetao)| \\
	&=n^{1/2} v^{-1}|\fa^T\Psi (\Psi -\hLpaa{-}\Psi\fS\Psi)(\Psi\hbeta-\Psi\fbetao)| \\
	&\leq  n^{1/2} v^{-1} \onenorm{\Psi \fa}{}\maxnorm{\Psi -\hLpaa{-} \Psi\fS\Psi} \onenorm{\Psi\hbeta-\Psi\fbetao}{}\\
	&\leq n^{1/2} \tilde a^{1/2} O(\tau_a)O_p(a \sqrt {\frac{\log p}{n}})
	= o_p(1),
	\end{align*}
	provided assumption $\tau_a = O(n^{-1/2}\tilde a^{-1/2})$. The order in
	\eqref{lemma.a3} in Lemma \ref{lemma.ac} gives
	\begin{align*}
	|T_{2}| =& n^{1/2} v^{-1}|\fa^T(\Psi -\hLpaa{-}\fS ) \Psic\fbetao| \\
	=& n^{1/2} v^{-1}|\fa^T(\hLpaa{-} \Psi\fS \Psic) \Psic\fbetao|\\
	\leq& n^{1/2} v^{-1}\twonorm{\hLpaa{-}}{}\twonorm{\hLpca{T}}{} \twonorm{\Psic\fbetao}{} \\
	\leq& n^{1/2} v^{-1}b_0^{-1} \twonorm{\hLpca{T}}{} \twonorm{\Psic \fbetao}{}
	= o_p(1),
	\end{align*}
	where the last inequality follows from \eqref{lemma.a3}.
	For $T_{3}$,
	\begin{align*}
	|T_{3}| =& n^{1/2} v^{-1}|\fa^T(\Psic - \hLpcc{-}\fS + \hLpcc{-}\hLpca{}\hLpaa{-}\fS)(\Psi\hbeta-\Psi\fbetao)|, \\
	=& n^{1/2} v^{-1} |\fa^T (\Psic  - \hLpcc{-} \hLpca{}
	+ \hLpcc{-}\hLpca{}\hLpaa{-} \Psi\fS \Psi )(\Psi\hbeta-\Psi\fbetao)| \\
	=& n^{1/2} v^{-1}|\fa^T \hLpcc{-} \hLpca{} (\Psi -\hLpaa{-}\Psi\fS\Psi) (\Psi\hbeta-\Psi\fbetao)| \\
	\leq & n^{1/2} v^{-1} \onenorm{\fa^T  \hLpcc{-} \hLpca{} }{} \maxnorm{\Psi -\hLpaa{-}\Psi\fS\Psi}\onenorm{\Psi\hbeta-\Psi\fbetao}{} \\
	\leq & n^{1/2} v^{-1}\tilde a ^{1/2} \twonorm{\hLpcc{-} }{} \twonorm{\hLpca{}}{} \maxnorm{\Psi -\hLpaa{-}\Psi\fS\Psi}\onenorm{\Psi\hbeta-\Psi\fbetao}{}\\
	\leq & n^{1/2} \tilde a ^{1/2} \tau_c^{-1} \twonorm{\hLpca{}}{} O(\tau_a) \onenorm{\Psi\hbeta-\Psi\fbetao}{} \\
	= & O_p(n^{1/2}\tilde a^{1/2} \tau_c^{-1} \rhomax{\hLpca{}}{}\tau_a a \sqrt{\frac{\log p}{n}}) = o_p(1),
	\end{align*}
	where $\onenorm{\fa^T  \hLpcc{-} \hLpca{} }{} \leq \tilde a^{1/2}
	\twonorm{\fa^T  \hLpcc{-} \hLpca{} }{} \leq \tilde a^{1/2} \tau_c^{-1} \twonorm{\hLpca{}}{}$ since $\fa^T  \hLpcc{-} \hLpca{}$
	is a $1 \times p$ vector with $\tilde a$ nonzero elements specified by $\Psi$.
	
	For $T_{4}$, Lemma \ref{lemma.ac} implies
	\begin{align*}
	|T_{4}| =& n^{1/2} v^{-1} |\fa^T(\Psic - \hLpcc{-}\fS + \hLpcc{-}\hLpca{}\hLpaa{-}\fS)(\Psic\hbeta-\Psic\fbetao)|\\
	=& n^{1/2} v^{-1} |\fa^T(\Psic - \hLpcc{-} \Psic\fS\Psi - \hLpcc{-} \Psic\fS\Psic \\
	&+ \hLpcc{-}\hLpca{}\hLpaa{-} \Psi\fS \Psi + \hLpcc{-}\hLpca{}\hLpaa{-} \Psi\fS \Psic )(\Psic\hbeta-\Psic\fbetao)|\\
	=& n^{1/2} v^{-1}|\fa^T(\Psic  - \hLpcc{-} \Psic\fS\Psic
	+ \hLpcc{-}\hLpca{}\hLpaa{-} \Psi\fS \Psic )\Psic\fbetao|\\
	=& n^{1/2} v^{-1} \twonorm{ \fa^{T}\hLpcc{-}(\Psic \tau_c + \hLpca{}\hLpaa{-} \hLpca{T})}{} \twonorm{\Psic\fbetao}{}\\
	\leq& n^{1/2} \tau_c^{-1}O(\tau_c + b_0^{-1} \rhomax{\hLpca{}}^2) \twonorm{\Psic\fbetao}{} \\
	= & n^{1/2} O(b_0^{-1}\tau_c\rhomax{\hLpca{}}^2) \twonorm{\Psic\fbetao}{}
	= o_p(1).
	\end{align*}
	
	Let $T_{5,*} =n^{-1/2} v_*^{-1}\fa^T  \hLs{-1} \fX^T \fepsilon$. Since
	$T_5-T_{5,*} = o_p(1)$ by assumption \ref{con.c1} and \ref{con.c2}, assumption \ref{con.a0} implies
	$T_{5,*}$ follows a standard normal distribution $N(0,1)$.
\end{proof}

\fi

\subsection{Proof of Corollary \ref{coro.2}}
\label{proof.coro2}
\begin{proof}
  Using similar arguments in Lemma \ref{lemma.fnormt}, we know the minimal
  variance of estimator $\hbetaat{\ftaua}$ satisfies
  \[\underset{1\leq i \leq a_*}{\min} \fe_i^T \ifSigmaaa \fSaa \ifSigmaaa \fe_i  \geq \rhomin{\ifSigmaaa \fSaa \ifSigmaaa }
	\geq \frac{\rhomin{\fSaa}}{(\rhomin{\fSaa} + \tau_a)^2}.
  \]
  It is easy to verify that
  \begin{align*}
    \sigma^2 \ifSigmacc \fScc \ifSigmacc
    \succ \sigma^2 \big [\fL^{-1} \fS (\fL^{-1})^T \big ]_{\mhA^c\mhA^c}.
  \end{align*}
  Consequently, we can prove the result by assessing the diagonal
  entries of $\ifSigmacc \fScc \ifSigmacc$.  The maximal variance of estimator
  $\hbetact{\ftauc}$ is bounded by
  \begin{align*}
  \underset{1\leq i \leq p-a_*}{\max} (\fe_i^{\bot})^T \ifSigmacc \fScc \ifSigmacc \fe_i^{\bot}
  \leq \rhomax{\ifSigmacc \fScc \ifSigmacc}
  \leq \frac{\rhomax{ \fScc}}{\tau_c^2}.
  \end{align*}
  Therefore, assumptions for $\tau_a$ and $\tau_c$ in Theorem \ref{thm.normal}
  imply
  \begin{align*}
    \underset{1\leq i \leq a_*}{\min} \fe_i^T \ifSigmaaa \fSaa \ifSigmaaa \fe_i
    \geq \frac{c_1}{ \rhomin{\fSaa } }
    \geq  \frac{c_2}{\rhomax{\fSaa}} \geq \\
    \underset{1\leq i \leq p-a_*}{\max} (\fe_i^{\bot})^T \ifSigmacc
    \fScc  \ifSigmacc \fe_i^{\bot},
  \end{align*}
  where $c_1$ and $c_2$ are two positive constants.

\end{proof}

\subsection{Proof of Corollary \ref{lab.coro2}}
\label{lab.proof.coro2}
\begin{proof}
  Assumptions \ref{assp.2} and \ref{lab.assp.7} and conditions for $\tau_a$ and
  $\tau_c$ imply that
  on $\mhA$ there exists
  \begin{align*}
	  \twonorm{\hbetaat{\ftaua}- \fbetaat{}}{} &\leq \twonorm{\ifSigmaaa
    }{} \frac{1}{n} \twonorm{\fX_{\mhA}^T\fepsilon}{} + \twonorm{r_a}{} \\
	  &\leq \twonorm{\ifSigmaaa }{}  \sqrt {\tilde a} \frac{1}{ n} \supnorm{
    \fX_{\mhA}^T\fepsilon}{} + o_p(1/\sqrt n)\\
     &\leq \frac{ O_p(\sqrt {\tilde a \log \tilde a/n})}
	  {\lambda_{\min}(\tilde a) + \tau_a} + o_p(1/\sqrt n) = O_p(\sqrt {\tilde a \log \tilde a/n}).
  \end{align*}
  Similarly on $\mhA^c$, based on the same assumptions, we obtain
  \begin{align*}
    \twonorm{\hbetact{\ftauc} - \fbetact}{} \leq& \twonorm{\ifSigmacc
    }{} \frac{1}{n} \twonorm{\fX_{\mhA^c}^T\fepsilon - \fSca \ifSigmaaa
    \fX_{\mhA}^T\fepsilon}{}  + \twonorm{r_c}{} \\
    = & \twonorm{\ifSigmacc}{} \frac{1}{n} \twonorm{\fX_{\mhA^c}^T(\fI_n -
        \frac{1}{n} \fX_{\mhA} \ifSigmaaa
	  \fX_{\mhA}^T)\fepsilon}{}  + o_p(1/\sqrt n)\\
    \leq & \twonorm{\ifSigmacc}{} \frac{\lambda_{\max}(\tilde a)}{
    \lambda_{\max}(\tilde a) + \tau_a} \sqrt{ p - \tilde a}\frac{1}{n} \supnorm{
    \fX_{\mhA^c}^T\fepsilon}{} + o_p(1/\sqrt n)\\
    \leq & \frac{1}{\tau_c} \frac{\lambda_{\max}(\tilde a)}{
	    \lambda_{\max}(\tilde a) + \tau_a} \sqrt{(p - \tilde a)\log
    (p-\tilde a)/n}  + o_p(1/\sqrt n)\\
	  =& o_p(\max \{1/\sqrt n, \sqrt{(p - \tilde a)\log
    (p-\tilde a)/n}/ \tau_c\}).
  \end{align*}
\end{proof}

\subsection{Proof of Theorem \ref{thm.wald}}
\label{proof.thm.wald}
\begin{proof}
Let $M_n = n\hbetag^T \hbetag - \sigma^2\tr{\fSigmagg}$.  Theorem
\ref{thm.normal} implies that $\sqrt n (\hbetag-\fbetag)
\overset{d}{\rightarrow} N(0, \sigma^2\fSigmagg^*)$, which further indicates $E
M_n\rightarrow 0$ given assumptions in Theorem \ref{thm.wald}. Furthermore, we
can verify that $\text{var}(M_2) =2 \sigma^4 \text{tr}\{  (\fSigmagg^*)^2\}$.
Applying the same arguments given by \citet{bai1996}, we can show $W_{bs}$
converges in distribution to $N(0,1)$ as $n \rightarrow \infty$.
\end{proof}


\bibliographystyle{imsart-nameyear}
\bibliography{References_5}

\end{document}